\documentclass[sigconf]{acmart}
\AtBeginDocument{%
  }
\usepackage{amsmath,algorithm, algorithmicx, algpseudocode,bbm}
\usepackage{textcomp}
\usepackage{amsfonts}
\usepackage{xcolor}
\usepackage{amsthm}
\usepackage[normalem]{ulem}
\usepackage{graphicx}
\usepackage{url}
\usepackage{bbm}
\usepackage{multirow}
\usepackage{hyperref}
\usepackage{booktabs}
\usepackage{enumitem}
\usepackage{subcaption}
\newtheorem{theorem}{Theorem}
\newtheorem{definition}{Definition}

\newtheorem{remark}{Remark}
\newtheorem{lemma}[theorem]{Lemma}
\newtheorem{assumption}{Assumption}
\usepackage{todonotes}
\definecolor{violet}{rgb}{0,0,0}
\definecolor{blue}{rgb}{0,0,0}
\definecolor{brown}{rgb}{0,0,0}
\definecolor{red}{rgb}{0,0,0}
\definecolor{green}{rgb}{0,0,0}
\definecolor{Tan}{rgb}{0,0,0}
\definecolor{cyan}{rgb}{0,0,0}
\definecolor{magenta}{rgb}{0,0,0}
\definecolor{purple}{rgb}{0,0,0}
\definecolor{alizarin}{rgb}{0.82, 0.1, 0.26}

\definecolor{olive}{rgb}{0,0,0}
\definecolor{alizarin}{rgb}{0,0,0}
\definecolor{orange}{rgb}{0,0,0}
\newcommand{\haoran}[1]{\textcolor{Tan}{#1}}
\newcommand{\rachid}[1]{\textcolor{blue}{#1}}

\setcopyright{acmlicensed}
\copyrightyear{2025}
\acmYear{2025}
\acmDOI{XXXXXXX.XXXXXXX}
\acmConference[Submission]{Make sure to enter the correct
  conference title from your rights confirmation email}{April 2025}{Pittsburgh, USA}
\acmISBN{978-1-4503-XXXX-X/2018/06}

\setcopyright{none}
\renewcommand\footnotetextcopyrightpermission[1]{}

\begin{document}

\title{Towards Optimal Heterogeneous Client Sampling in Multi-Model Federated Learning}

\author{Haoran Zhang}
\email{haoranz5@andrew.cmu.edu}
\affiliation{%
  \institution{Carnegie Mellon University}
  \city{Pittsburgh}
  \state{PA}
  \country{USA}
}

\author{Zejun Gong}
\email{zejung@alumni.cmu.edu}
\affiliation{%
  \institution{Carnegie Mellon University}
  \city{Pittsburgh}
  \state{PA}
  \country{USA}
}

\author{Zekai Li}
\email{zekail@alumni.cmu.edu}
\affiliation{%
  \institution{Carnegie Mellon University}
  \city{Pittsburgh}
  \state{PA}
  \country{USA}
}

\author{Marie Siew}
\email{marie\_siew@sutd.edu.sg}
\affiliation{%
  \institution{Singapore University of Technology and Design}
  \country{Singapore}
}

\author{Carlee Joe-Wong}
\email{cjoewong@andrew.cmu.edu}
\affiliation{%
  \institution{Carnegie Mellon University}
  \city{Pittsburgh}
  \state{PA}
  \country{USA}
}

\author{Rachid El-Azouzi}
\email{rachid.elazouzi@univ-avignon.fr}
\affiliation{%
  \institution{University of Avignon}
  \city{Avignon}
  \country{France}
}

\begin{abstract}
Federated learning (FL) allows edge devices to collaboratively train models without sharing local data. As FL gains popularity, clients may need to train multiple unrelated FL models, but communication constraints limit their ability to train all models simultaneously. 
While clients could train FL models sequentially, opportunistically having FL clients concurrently train different models---termed multi-model federated learning (MMFL)---can reduce the overall training time. Prior work uses simple client-to-model assignments that do not optimize the contribution of each client to each model over the course of its training. 
Prior work on single-model FL shows that intelligent client selection can greatly accelerate convergence, but na\"ive extensions to MMFL can violate heterogeneous resource constraints at both the server and the clients.
In this work, we develop a novel convergence analysis of MMFL with arbitrary client sampling methods, theoretically demonstrating the strengths and limitations of previous well-established gradient-based methods. 
Motivated by this analysis, we propose MMFL-LVR, a loss-based sampling method that minimizes training variance while explicitly respecting communication limits at the server and reducing computational costs at the clients.
We extend this to MMFL-StaleVR, which incorporates stale updates for improved efficiency and stability, and MMFL-StaleVRE, a lightweight variant suitable for low-overhead deployment.
Experiments show our methods improve average accuracy by up to 19.1\% over \textit{random} sampling, with only a 5.4\% gap from the theoretical optimum (full client participation).\footnote{Code: \href{https://github.com/Researcher-MMFL/MMFL-OptimalSampling}{https://github.com/Researcher-MMFL/MMFL-OptimalSampling}} 
\end{abstract}




\maketitle

\section{Introduction}

Federated Learning (FL) \cite{mcmahan2017communication} is an increasingly popular distributed learning paradigm that enables clients to collaboratively train a deep learning model while keeping their (often private) local data local. 
Specifically, in an FL system, a central server maintains a global model and periodically receives updates of model weights from clients within the system. The server aggregates these updates and sends the global model back to the clients to resume local training. 
In this paper, we primarily consider clients that are edge devices, such as smartphones or IoT (Internet-of-Things) devices.

Most FL research assumes each client trains only one model. However, in practice some clients may need to train multiple FL models. For example, a smartphone might act as a client in training \haoran{or fine-tuning} Google keyboard prediction \cite{hard2018federated}, keyword-spotting \cite{shah2020training}, speech recognition \cite{guliani2021training}, and more. 
A simple solution to this problem is to train one model after another sequentially, moving to the next model when the current one reaches either a predefined accuracy threshold or a fixed training time limit. However, such an approach forces the later models to wait for all prior models to train, which may be unfair and causes the total training time to scale linearly with the number of models~\cite{askin2024fedast}. Allowing models to be trained concurrently, which we call \textit{multi-model FL} (MMFL), allows clients to opportunistically train different models and has been shown in some prior works to accelerate overall training~\cite{askin2024fedast, bhuyan2022multi, siew2023fair, siew2024fair,atapour2023multi,liu2022multi,chang2024asynchronous,liu2023multi}.
\rachid{
However, MMFL encounters significant communication bottlenecks due to the requirement to transfer gradients for all tasks in each round. Deep neural networks (DNNs) with millions of parameters generate updates ranging from megabytes to gigabytes \cite{lin2018deep}. As new AI applications increasingly demand larger models, a critical question emerges: How can parallel learning be effectively orchestrated within limited communication bandwidth, even when some clients can train multiple models simultaneously? This communication constraint becomes particularly acute when implementing MMFL on edge devices to train multiple models in parallel.}

%
\textcolor{olive}{
Previous work on MMFL generally overlooks the heterogeneous capabilities of clients to contribute to different models as training progresses, and often assumes each client can only train one model per round due to computation or communication limits \cite{askin2024fedast, bhuyan2022multi, siew2023fair, siew2024fair,atapour2023multi,liu2022multi,chang2024asynchronous,liu2023multi}. In practice, clients may be able to train more than one model in parallel and contribute more effectively to certain models than others. This highlights the need to dynamically identify the most relevant subset of clients at each training stage, taking into account both their computational capacity and local data distribution.}
Our goal in this paper is to propose an MMFL training framework that \textbf{intelligently accounts for such client contributions},  while respecting the communication and computation constraints of the clients,  \rachid{with the aim of obtaining significant reductions in the overall communication costs. }

\subsection{Research Challenges in MMFL Systems}

Prior work on single-model FL (SMFL) has explored intelligent client selection, which aims to reduce clients' communication and computation burden by allowing them to infrequently participate in the model training. We can then optimize the probability that each client contributes to the model at a given time based on the client's estimated ability to reduce the training loss at the current point in the training \cite{chen2020optimal,wang2024delta}. Adapting such methods to MMFL is then a promising approach to achieving our goal of intelligent client selection while respecting communication constraints. Extending these client selection methods to MMFL, however, raises the following research challenges:

\textbf{Heterogeneity in client communication constraints.} 
\textcolor{olive}{While prior work \cite{bhuyan2022multi,siew2023fair,siew2024fair,zhangposter,askin2024fedast} typically assumes uniform communication and computation capabilities across clients—often limiting each client to handling one model update per round—real-world scenarios deviate from this assumption. }
Different clients may have access to different amounts of uplink bandwidth, e.g., clients on an expensive cellular roaming data plan may wish to limit the number of model updates they must send to the server, while those on a high-quality WiFi connection may have no such constraints. These communication limitations necessarily \textit{couple the selection of clients for each concurrently trained model}: asking a communication-constrained client to train one model would preclude that client from training other models in this training round. We must therefore \textit{jointly optimize} client sampling so as to respect these constraints. To the best of our knowledge, \textit{we are the first to consider heterogeneity in these MMFL client constraints. 
}
\textcolor{olive}{Specifically, in this paper, we allow clients to handle multiple training tasks in parallel, assuming their diverse communication and computation abilities, and optimize the training process given the contratints.}

\textbf{Client computation constraints.} Some FL clients may have limited local training capabilities, e.g., embedded low-power sensors \cite{imteaj2021survey}. 
\textcolor{olive}{One track of works minimizes the update variance to stabilize the training \cite{chen2020optimal,wang2024delta,zhangposter}}
, however, in MMFL, solving the resulting optimization problem \textit{requires to know the gradient of each client for each model in a given training round}: in other words, clients must train all models in all rounds, in order for us to select which clients should send their local updates to the server. Such local computation burden, while perhaps feasible when training a single FL model, quickly becomes infeasible as the number of models in our MMFL setting increases. We therefore seek a variant of optimal client sampling that does not require model gradients to find the optimal sampling probabilities. 

\textbf{Increased training variance \textcolor{orange}{from multi-model heterogeneity}.} In the SMFL setting, algorithms that employ selective client participation often compensate for the resulting model bias by scaling their updates accordingly~\cite{chen2020optimal,wang2024delta}. However, such compensation has recently been shown to bring additional variance in the training process, which may be mitigated by incorporating stale updates from previous training rounds into the model aggregation at the server~\cite{jhunjhunwala2022fedvarp, rodio2024fedstale}. We expect MMFL to exacerbate this increase in variance: since clients distribute their training capabilities across multiple models, their ability to contribute to training any one of these models is significantly more volatile than in SMFL, especially as the number of models increases. Existing methods for incorporating stale updates, which focus on statically weighting these updates in the model aggregation regardless of their relevance to the current model, may then be insufficient to stabilize the MMFL training. 
\rachid{We therefore seek a method to further stabilize training by dynamically incorporating stale updates.}
\subsection{Our Contributions}

Given the challenges of building an effective MMFL system with intelligent client sampling that respects heterogeneous client communication constraints, we outline related work in Section~\ref{sec:relatedWork} and then make the following contributions: 
\begin{enumerate}
\item We first assume a given client-model assignment strategy and \textbf{analyze the resulting convergence of MMFL} (Section \ref{sec:SysModel}). 
As assigning different clients to models in MMFL may introduce training bias, we adjust our model aggregation to ensure unbiased convergence. 
Our theoretical analysis reveals how client sampling strategies impact training stability, going beyond existing results in analyzing client sampling for a single FL model. 

\item 
Through our analysis, we theoretically demonstrate the advantages of gradient-based optimal sampling methods discussed in prior work \cite{chen2020optimal,wang2024delta,zhangposter}, \rachid{while also revealing their computational and communication overheads, as well as overlooked factors in their optimization objectives. To address these limitations, we propose MMFL-LVR (Section \ref{sec:variance-reduce}), which minimizes the variance of the surrogate objective per round while strictly adhering to the system's communication constraints, thereby addressing our first research challenge.
Unlike gradient-based methods, MMFL-LVR leverages loss values to design optimal sampling strategies, tackling the second research challenge. 
} 


\item 
\rachid {To further stabilize training and mitigate the impact of participation variance caused by client sampling methods, particularly under communication constraints that limit the number of clients per training round, we propose MMFL-StaleVR. This method optimally reuses stale updates from clients in model aggregation (Section \ref{sec:varianceH}). Unlike previous approaches that apply a single global coefficient to all stale updates~\cite{jhunjhunwala2022fedvarp, rodio2024fedstale}, MMFL-StaleVR dynamically adjusts the contribution of each client's most recent update based on its effectiveness in stabilizing training. This approach directly addresses the first and third research challenges. }
\item \rachid{MMFL-StaleVR requires all clients to perform training in order to optimally reuse stale updates during model aggregation. To address this limitation, we further develop an estimate of the optimal solution for stale updates in model aggregation (MMFL-StaleVRE, Section \ref{sec:varianceH}), which requires only selected clients to perform training, thus achieving all three research challenges. }

\item \textcolor{black}{We conduct \textbf{extensive experiments} in various settings on real-world datasets, demonstrating that our methods significantly outperform simple extensions of SMFL baselines as well as existing MMFL methods. By adopting aggregation with stale updates to mitigate the impact of participation heterogeneity, MMFL-StaleVR achieves 94\% of the theoretical best performance \textcolor{brown}{under the same experiment settings} (Section \ref{sec:exp}). } 
\end{enumerate}
We conclude the paper in Section~\ref{sec:conclusion}. 
Due to space limitations, all proofs are in the Supplementary Material (Section \ref{sec:app}).

\section{Related Work}
\label{sec:relatedWork}

\textbf{Multi-Model Federated Learning.}
\textcolor{black}{
Multi-model FL was first proposed in \cite{bhuyan2022multi}. 
Many subsequent studies \cite{bhuyan2022multi, siew2023fair, siew2024fair, liu2022multi, chang2024asynchronous,zhangposter, atapour2023multi,liu2023multi,askin2024fedast} have sought to improve the training performance of all models by effectively allocating the available training resources (clients) across models. 
This resource allocation problem can be solved via multi-armed bandits \cite{bhuyan2022multi}, using reinforcement learning \cite{liu2022multi,atapour2023multi}, or directly constructing the probability distribution of training each model \cite{siew2023fair,siew2024fair}. 
Other works~\cite{chang2024asynchronous,askin2024fedast} address variations in model complexity and local training times by proposing asynchronous MMFL training schemes. 
While these works improve the average or minimum accuracy across MMFL models, they generally overlook heterogeneity in clients' resources \textcolor{red}{and dataset distributions}. 
Client heterogeneity in MMFL introduces new
client-model allocation challenges because: 
1) Clients impose different communication costs at the server, as clients that train more models must communicate all of their updates to the server; and 2) The system may converge towards clients with more powerful communication or computational abilities since they can contribute to more models per round, causing convergence bias. 
}



\textbf{Single-Model Client Sampling.} 
Allocating clients to models in MMFL can be viewed as sampling a set of clients to train each model in each round, generalizing the
\textcolor{black}{
well-studied client sampling in single-model FL \cite{cho2020client, luo2022tackling, wang2024delta, chen2020optimal,fraboni2021clustered,ruan2024valuable}. These strategies generally define metrics that evaluate the ``importance'' of clients to the training task and prioritize sampling accordingly. 
Some na\"ive methods \cite{cho2020client} can be easily adapted for MMFL settings, while others \cite{luo2022tackling, wang2024delta, chen2020optimal,fraboni2021clustered,ruan2024valuable} cannot.
Na\"ive methods \cite{cho2020client} rank clients based on their importance metrics and select the top ones. 
However, they may 
introduce training bias, as some clients are selected more frequently than others, leading to a non-vanishing bias in the final model \cite{cho2020client, cho2022towards}. 
To address sampling bias, more sophisticated methods \cite{luo2022tackling, wang2024delta, chen2020optimal,fraboni2021clustered} 
also adjust 
the aggregation coefficients according to their constructed client sampling distribution, theoretically avoiding bias. However, these methods cannot be directly applied to MMFL because they generate only the probability of sampling a client; MMFL also requires determining which model a client should train given heterogeneous client resources.
Client-model assignment distributions in previous MMFL works \cite{siew2023fair, siew2024fair,askin2024fedast} 
are uniform across all clients, ignoring client heterogeneity and leaving significant room for improvement.
\textcolor{olive}{\cite{zhangposter} accounts for dataset heterogeneity across clients but overlooks differences in computational capabilities. Moreover, it requires collecting gradients from all clients to form the distribution, which is costly in MMFL settings. 
}
}

\section{\textcolor{brown}{System Model}} 
\label{sec:SysModel}
\textcolor{brown}{
In MMFL (see Fig. \ref{fig:overview}), we assume \textcolor{red}{$S$} models are trained iteratively across global rounds indexed by $\tau=1,2,\ldots, T$. Within each global round $\tau$, the server allocates clients to each model $s = 1,2,\ldots,S$. 
Clients then train each of their assigned models locally, as in SMFL. After the local training, clients upload their local model updates to the server, where they are aggregated for each model, \textcolor{red}{and the process repeats.} }

\begin{figure}
    \centering
    \includegraphics[width = 0.48\textwidth]{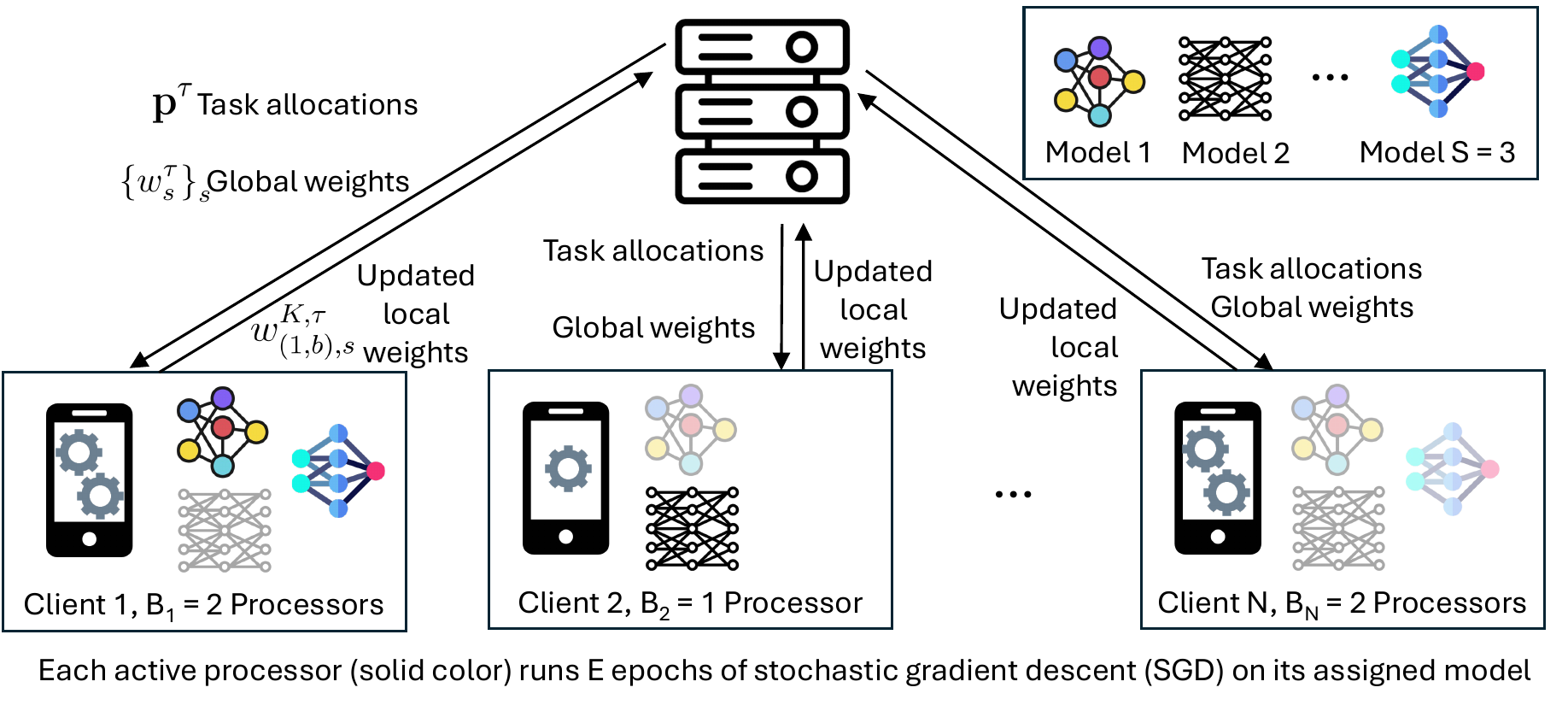}
    \caption{Overview of the MMFL system for an example with $S = 3$ models. In each global round $\tau$, the server probabilistically assigns models to a subset of processors at the FL clients. Models that each client has the data to train are shown at each client, and faded models indicate ones that have not been assigned in this training round.}
    \label{fig:overview}
\end{figure}

\subsection{\textcolor{brown}{Problem Formulation}}
Similar to previous MMFL works \cite{bhuyan2022multi, siew2023fair, siew2024fair, askin2024fedast, liu2022multi, chang2024asynchronous,liu2023multi,atapour2023multi}, we consider an MMFL system with $N$ clients and $S$ unrelated models. \textcolor{brown}{Define $\mathcal{S}$ as the set of models}.  
We model clients' \emph{heterogeneous communication and computational abilities} by assuming that each client $i$ can train $B_i$ models per global training round. For ease of description, we say that client $i$ has $B_i$ \textbf{processors} to handle FL training tasks\textcolor{brown}{, where a ``training task'' means performing local training for a model and sending the result to the server within one global round. 
Thus, \textbf{$B_i$ represents the smaller of two limits}: the maximum number of models a client can send to the server within a single round, \textbf{given communication constraints}, and the maximum number of models the client can train in parallel during that round, \textbf{given computation constraints}.
Define $\mathcal{B}_i$ as the set of processors of client $i$.\footnote{In real applications, each training task need not be assigned to a dedicated client processor; we use the concept of a \textit{processor} as a useful abstraction.} 
This model generalizes that of previous MMFL works \cite{bhuyan2022multi, siew2023fair, siew2024fair, askin2024fedast, liu2022multi, chang2024asynchronous,liu2023multi,atapour2023multi}, which assume a client can only train one model per round, i.e., $B_i = 1$ for all clients $i$.}
Since MMFL includes multiple unrelated FL training tasks within one system, some clients \emph{may lack the datasets for specific models}, \textcolor{red}{which further complicates resource allocation.
} We define the set of clients available for model $s$ as $\mathcal{N}_{s}$ and the set of models available for client $i$ as $\mathcal{S}_i$.  
\textcolor{brown}{We define the \textbf{objective for each model $s$} as: }
\begin{align}
\min_{w_s}F_s= \min_{w_s} \sum_{i \in \mathcal{N}_s} d_{i,s} f_{i,s}(w_s).
\label{eq:ModelS_obj}
 \end{align}
\textcolor{black}{
where $w_s$ denotes the parameters of model $s$.  
The local objective $f_{i,s}(w_s)$ is defined as the empirical risk over local data: 
$f_{i,s}(w_s)=\frac{1}{n_{i,s}} \sum_{\xi\in\mathcal{D}_{i,s}}l(w_s,\xi)$, where $\mathcal{D}_{i,s}$ is the set of datapoints in client $i$ that can be used to train model $s$, and $n_{i,s} = \left|\mathcal{D}_{i,s}\right|$ is the number of datapoints that client $i$ can use to train model $s$. The loss function $l$ evaluates model performance on one datapoint, e.g., with cross-entropy loss. }
\textcolor{black}{$d_{i,s}=n_{i,s}/\sum_{j\in \mathcal{N}_s} n_{j,s}$ denotes the fraction of client $i$'s dataset size relative to the total dataset size for model $s$.} 
Based on (\ref{eq:ModelS_obj}), 
the \textbf{objective of the MMFL system} is to minimize the sum of the objectives of the $S$ models, i.e.,
\begin{align}
\min_{w_1, \cdots, w_S}F=\min_{w_1, \cdots, w_S} \sum_{s=1}^S \sum_{i\in \mathcal{N}_s} d_{i,s} f_{i,s}(w_s)
\label{eq:MMFLsystemObj}
\end{align}

\subsection{MMFL Training Procedure}
\label{sec:MMFLsteps}
\textcolor{black}{Before the local training in each round, the server will firstly assign training tasks to clients.  The task allocation process consists of assigning 
$c$ ($c\leq B_i$) training tasks to each client $i$ in each round.}
We define $p_{s|(i,b)}^\tau$ as the probability of assigning processor $b$ at client $i$, which we denote as processor $(i,b)$, to train model $s$ in global round $\tau$ and send its result to the server. 
Thus, the server samples from the distribution $\mathbf{p}^\tau=\{p_{s|(i,b)}^\tau\}_{s\in \mathcal{S},i\in \mathcal{N}_s, b\in\mathcal{B}_i} $ in assigning tasks to client processors. \rachid{This probability distribution is designed by the server (which we will optimize in Section~\ref{sec:variance-reduce}) and plays a crucial role in accelerating convergence under communication constraints.}
After the training task allocation, active clients execute the assigned tasks and send updates to the server for aggregation. 
The detailed steps in each global round $\tau$ are described below: \\
\textbf{Training Task Allocation:} Given the probability distribution $\mathbf{p}^\tau$, the server generates the sets of participating processors for each model $s$,  $(\mathcal{A}_{\tau,s})_{s\in \mathcal{S}}$ as follows: for each $s\in \mathcal{S}$ and $i\in \mathcal{N}_s$, 
\textcolor{orange}{processor} $(i,b)\in \mathcal{A}_{\tau,s}$ \textcolor{orange}{trains it} with probability  $p_{s|(i,b)}^\tau$. The assignment of each processor to a task $s$ is independent of other processors.\\
\textbf{Synchronization:} Each processor $(i,b)\in \mathcal{A}_{\tau,s}$, $s\in \mathcal{S}$,  initializes the model weights with the global model $s$ weights $w_s^\tau$: $w_{(i,b),s}^{1,\tau} = w_s^{\tau}$, where $w_{(i,b),s}^{1,\tau}$ denotes processor $(i,b)$'s initial weights of model $s$ in round $\tau$. \\
\textbf{Local Training:} Each processor $(i,b)\in \mathcal{A}_{\tau,s}$, $s\in \mathcal{S}$, performs local training by running mini-batch stochastic gradient descent for $K$ local epochs:  
 $w_{(i,b),s}^{t+1,\tau} = w_{(i,b),s}^{t,\tau} - \eta_{\tau,s} \nabla f_{i,s}(w_{(i,b),s}^{t,\tau}, \xi_{i,s}^{t,\tau})$ for $t = 1,2,\ldots,K$, where   $\xi_{i,s}^{t,\tau}$ represents a random mini-batch of datapoints from local data distribution $\mathcal{D}_{i,s}$ and $\eta_{\tau,s}$ the local learning rate. 
Define the change in model weights produced by processor $(i,b)$ as $G_{(i,b),s}^\tau = \eta_{\tau,s} \sum_{t=1}^{K}\nabla f_{i,s}(w_{(i,b),s}^{t,\tau},\xi_{i,s}^{t,\tau})$.\\
\textbf{Aggregation:} 
Clients \textcolor{brown}{send} their local updates to the server, which aggregates the weights for each model $s$:
\begin{align}
w_{s}^{\tau+1} &= w_s^{\tau} - \sum_{(i,b) \in \mathcal{A}_{\tau,s}} P_{(i,b),s}^\tau G_{(i,b),s}^\tau\label{eq:aggregation}
\end{align}
where $P_{(i,b),s}^\tau = \frac{d_{i,s}}{B_i p_{s|(i,b)}^\tau}$. 
The aggregation coefficient $P_{(i,b),s}^\tau$ acts as a scaling factor to guarantee an unbiased estimator for full participation training, i.e., all clients train all models:
\begin{align}
&\mathbb{E}\left[\sum_{(i,b) \in \mathcal{A}_{\tau,s}} P_{(i,b),s}^\tau G_{(i,b),s}^\tau\middle|\; G_{(i,b),s}^\tau\right] \\
&= \sum_{i \in \mathcal{N}_s} \sum_{b=1}^{B_i}  \frac{  \mathbb{E}\left[\mathbbm{1}_{(i,b) \in \mathcal{A}_{\tau,s}} \right] d_{i,s} G_{(i,b),s}^\tau}{B_i p_{s|(i,b)}^\tau} 
= \sum_{i \in \mathcal{N}_s} \sum_{b=1}^{B_i} \frac{d_{i,s}}{B_i} G_{(i,b),s}^\tau
\label{eq:fullparticipation}
\end{align}
Here $\mathbbm{1}_{(i,b)\in\mathcal{A}_{\tau,s}}$ is an indicator of whether processor $(i,b)$ is allocated to 
model $s$ for training in round $\tau$. 
 We simplify this notation to 
$\mathbbm{1}_{(i,b)}^{s,\tau}$ in the following. 
\textcolor{black}{
The above expectation 
is taken over $\mathcal{A}_{\tau,s}$ (set of participating processors). Eq. \ref{eq:fullparticipation} is full participation update using mini-batch stochastic gradient descent (SGD). }\textit{FedAvg} \cite{mcmahan2017communication} can be viewed as a special case of our MMFL system with $S=1$ and $B_i=1$ for all clients. 
\begin{remark}[Independent processor sampling]\label{remark:ProcessorClarify}
\textcolor{brown}{Independent client sampling is widely adopted in single-model sampling methods \cite{chen2020optimal,luo2022tackling,wang2024delta,fraboni2021clustered}. In MMFL, with heterogeneous computational and communication abilities, independent sampling at the \textbf{processor level} helps manage participating clients given resource constraints and provably avoids training bias.}
Since we assume each processor's sampling is independent, $l$ processors ($1 \leq l \leq B_i$) in client $i$ could train the same model. In practice, client $i$ can train with one processor, and upload $l\cdot G_{(i,b),s}^\tau$ as its update to the server. 
\end{remark}

Before designing optimal sampling methods with optimized distribution $\mathbf{p}^{\tau}$, we study MMFL convergence and highlight dominant terms critical for accelerating convergence. 
These terms \textcolor{olive}{theoretically confirm the benefits of prior gradient-based optimal sampling methods \cite{chen2020optimal,wang2024delta,zhangposter}, while also revealing their potential drawbacks that may destabilize training---insights that directly inform the design of our proposed methods.}

\subsection{Convergence Analysis}
\label{sec:convergence}
We analyze the convergence of the MMFL system \textcolor{orange}{with heterogeneous client computation capabilities} described above, revealing how the server's choice of $\mathbf{p}^{\tau}$ influences the convergence speed. This analysis involves several assumptions. Assumptions \ref{assumption:L}-\ref{assumption:G} are standard in FL literature \cite{cho2020client,cho2022towards,cho2023convergence,ruan2020towards,rodio2024fedstale,jhunjhunwala2022fedvarp,chen2020optimal}. 
Assumption \ref{assumption:p} ensures a lower-bounded probability to prevent extreme $\mathbf{p}^\tau$ from causing some clients to disappear entirely from the training as is common in FL works~\cite{xiang2024efficient}. In Definition \ref{assumption:e}, we quantify clients' non-iid data distribution, which is common in real-world FL applications. 

\begin{assumption}[$L$-smoothness]\label{assumption:L}
Each $f_{i,s}$ is L-smooth, and thus $F=\sum_{s=1}^S \sum_{i\in \mathcal{N}_s} d_{i,s} f_{i,s}$ is also L-smooth.  
\end{assumption}
\begin{assumption}[Strong convexity]
Each $f_{i,s}$ is $\mu$-strongly convex, and thus $F=\sum_{s=1}^S \sum_{i\in \mathcal{N}_s} d_{i,s} f_{i,s}$ is as well.  
\end{assumption}
\begin{assumption}[Bounded variance]
The variance of the mini-batch gradients is bounded: 
\[\mathbb{E}_{\xi_{i,s}\sim \mathcal{D}_{i,s}} (\|\nabla f_{i,s}(w,\xi_{i,s}) - \nabla f_{i,s}(w)\|^2) \leq \sigma_{i,s}^2,\; \forall i,s.\] 
\end{assumption}
\begin{assumption}[Bounded gradients]\label{assumption:G}
The expected squared norm of the gradients is uniformly bounded: 
\[\mathbb{E}_{\xi_{i,s}\sim \mathcal{D}_{i,s}}(\|\nabla f_{i,s}(w,\xi_{i,s}) \|^2) \leq \bar{\sigma}^2,\; \forall i,s.\]
\end{assumption}


\begin{assumption}[Lower-bounded probability]\label{assumption:p}
There exists a lower bound $\theta>0$, such that $p_{s|(i,b)}^\tau \geq \theta$ for all $i,b,s,\tau$.
\end{assumption}

\begin{definition}[\textcolor{brown}{Non-iid data distribution}]\label{assumption:e}
We quantify the non-iid level of clients' data distribution as: 
\( \Gamma_{i,s} = f_{i,s}(w_s^*) - f_{i,s}(w_s^{*,i}),\) 
where \( w_s^* \) minimizes the objective $F_s$ and \( w_s^{*,i} \) minimizes $f_{i,s}$ for client \( i \). There exists a gap between the global optimal and the local optimal weights: \( \|w_s^* - w_s^{*,i}\| \geq e_w > 0 \), which also implies \( \|\nabla f_{i,s}(w_s^\tau)\| \geq e_f > 0 \) for all \( i, s, \tau\). 
\end{definition}

\begin{theorem}[Convergence] Let $w_s^*$ denote the optimal weights of model $s$. If the learning rate $\eta_{\tau,s}=\frac{16}{\mu } \frac{1}{(\tau+1)K+\gamma_{\tau,s}}$, then
\begin{align}
\mathbb{E}\left(\|w_{s}^{\tau} - w_s^*\|^2\right)\leq \frac{V_\tau}{(\tau K +\gamma_{\tau,s})^2} \label{eq:convergence}
\end{align}
We define
$\gamma_{\tau,s}=\max \{\frac{32L}{\mu},4K\sum_{i\in \mathcal{N}_s}\sum_{b=1}^{B_i} \mathbbm{1}_{(i,b)}^{s,\tau} P_{(i,b),s}^\tau\}$, \\
$V_\tau=\max\{\gamma_\tau^2 \mathbb{E}(\|w_s^0-w_s^*\|^2), (\frac{16}{\mu})^2\sum_{\tau'=0}^{\tau-1}z_{\tau'}\}$, \\
$z_{\tau'}=\mathbb{E}[Z_g^{\tau'}+Z_l^{\tau'}+Z_p^{\tau'}]$, with each term in $z_{\tau'}$ defined as:\\
\(
\mathbb{E}[Z_g^\tau]=K\sum_{i\in \mathcal{N}_s}\sum_{b=1}^{B_i} \frac{(\frac{d_{i,s}}{B_i}\sigma_{i,s})^2}{p_{s|(i,b)}^\tau}+4LK\sum_{i\in \mathcal{N}_s} d_{i,s} \Gamma_{i,s}\\
+\max(\frac{B_i}{d_{i,s}}) \mathbb{E}[\sum_{i\in \mathcal{N}_s}\sum_{b=1}^{B_i} \frac{(\frac{d_{i,s}}{B_i})^2 \sum_{t=1}^{K} \| \nabla {f}_{i,s} (w_{(i,b),s}^{t,\tau}) \|^2}{p_{s|(i,b)}^\tau}]\),\\
\(
\mathbb{E}[Z_l^\tau]=R \mathbb{E}[V_s  \sum_{i\in \mathcal{N}_s}\sum_{b=1}^{B_i} (\mathbbm{1}_{(i,b)}^{s,\tau}P_{(i,b),s}^\tau f_{i,s}(w_s^\tau)-\frac{d_{i,s}}{B_i}f_{i,s}(w_s^\tau))^2]\), where $R=\frac{2K^3\bar{\sigma}^2}{e_w^2e_f^2 \theta}, V_s=\sum_{i\in\mathcal{N}_s}B_i$, \\
\(
\mathbb{E}[Z_p^\tau]=(\frac{2}{\theta}+K(2+\frac{\mu}{2L})) K^2\bar{\sigma}^2 
+\frac{2K^3\bar{\sigma}^2}{\theta} \mathbb{E}[(\sum_{i\in \mathcal{N}_s}\sum_{b=1}^{B_i} \mathbbm{1}_{(i,b)}^{s,\tau}P_{(i,b),s}^\tau -1)^2]
\).
\label{thoerem:convergence}
\end{theorem}

\begin{remark}[Upper bound convergence to zero]
\haoran{From Assumption \ref{assumption:p}, it is evident that $\gamma_{\tau,s}$ in Eq.~\eqref{eq:convergence} is bounded. Thus, as the global round index $\tau$ increases, the denominator of Eq.~\eqref{eq:convergence} grows quadratically, while the numerator grows linearly. Consequently, $\lim_{\tau\rightarrow\infty}\mathbb{E}\left(\|w_{s}^{\tau} - w_s^*\|^2\right) \rightarrow 0$. 
}
\end{remark}

\textbf{Interpretation:} 
Considering $B_i$ processors for each client is analogous to having $B_i$ clients with identical datasets, with each client's contribution reduced by $\frac{1}{B_i}$. 
Equation~\eqref{eq:convergence}'s bound is mainly influenced by the terms $\mathbb{E}[Z_g^\tau]$, $\mathbb{E}[Z_l^\tau]$, and $\mathbb{E}[Z_p^\tau]$.

In $\mathbb{E}[Z_g^\tau]$, the first two terms reflect the influence of the non-iid data (through $\Gamma_{i,s}$) 
and the variance of the mini-batch gradient ($\sigma_{i,s}$) on the convergence upper bound. These terms are determined by datasets and local training batch size. If the data is very non-iid across clients, or the mini-batch gradient is inaccurate, $\mathbb{E}[Z_g^\tau]$'s value increases, slowing convergence speed. The third term, $\sum_{i\in \mathcal{N}_s}\sum_{b=1}^{B_i} \frac{(\frac{d_{i,s}}{B_i})^2 \sum_{t=1}^{K} \| \nabla {f}_{i,s} (w_{(i,b),s}^{t,\tau}) \|^2}{p_{s|(i,b)}^\tau}$, indicates that the relationship between the norm of local updates and sampling probability distribution could influence the convergence speed. 
\textcolor{brown}{This term reflects the scale of the \textbf{variance of sampled updates} during training, which can be expressed as:
\begin{align}
\sum_{s=1}^S \mathbb{E}\left[\left\|\sum_{(i,b)\in\mathcal{A}_{\tau,s}} \frac{P_{(i,b),s}^\tau G_{(i,b),s}^\tau}{\eta_{\tau,s}} - \sum_{i\in\mathcal{N}_s}\sum_{b=1}^{B_i} \frac{d_{i,s}}{B_i} \frac{G_{(i,b),s}^\tau}{\eta_{\tau,s}}\right\|^2
\right].\label{eq:OSvariance}
\end{align}
In Eq. \eqref{eq:OSvariance}, we note  that the expectation of the sampled update (left term) is the full participation update.
Using the fact that sampling at each processor is independent of other processors, we now show that the term $\sum_{i\in \mathcal{N}_s}\sum_{b=1}^{B_i} \frac{(\frac{d_{i,s}}{B_i})^2 \sum_{t=1}^{K} \| \nabla {f}_{i,s} (w_{(i,b),s}^{t,\tau}) \|^2}{p_{s|(i,b)}^\tau}$ reflects the scale of the variance of sampled updates:}
\begin{align}
&\mathbb{E}\left[\left\|\sum_{(i,b)\in\mathcal{A}_{\tau,s}}  \frac{P_{(i,b),s}^\tau G_{(i,b),s}^\tau}{\eta_{\tau,s}} - \sum_{i\in\mathcal{N}_s}\sum_{b=1}^{B_i} \frac{d_{i,s}}{B_i} \frac{G_{(i,b),s}^\tau}{\eta_{\tau,s}}\right\|^2\right] \label{eq:varianceRewriteBegin} \\
&= \sum_{i\in\mathcal{N}_s}\sum_{b=1}^{B_i} \frac{\|\frac{d_{i,s}}{B_i \eta_{\tau,s}} G_{(i,b),s}^\tau\|^2}{ p_{s|(i,b)}^\tau} - 
\sum_{i\in\mathcal{N}_s}\sum_{b=1}^{B_i} \left\|\frac{d_{i,s} G_{(i,b),s}^\tau}{B_i \eta_{\tau,s}}\right\|^2 \label{eq:varianceRewrite}
\end{align}
\haoran{Detailed steps to derive Eq. \eqref{eq:varianceRewrite} from Eq. \eqref{eq:varianceRewriteBegin} are illustrated in the Supplementary Material.}
The distribution $\mathbf{p}^{\tau}$ does not change the second term in Eq. \eqref{eq:varianceRewrite}, and $K\sum_{i\in \mathcal{N}_s}\sum_{b=1}^{B_i} \frac{(\frac{d_{i,s}}{B_i})^2  \sum_{t=1}^{K} \| \nabla {f}_{i,s} (w_{(i,b),s}^{t,\tau}) \|^2}{p_{s|(i,b)}^\tau}$ upper bounds the first term of Eq. \eqref{eq:varianceRewrite}. Thus, as we might intuitively expect, choosing the client sampling distribution $\mathbf{p}^{\tau}$ so as to minimize the variance of sampled updates as in Eq.~\eqref{eq:varianceRewrite} can accelerate the convergence and reduce the convergence bound in Eq.~\eqref{eq:convergence}.

\textcolor{olive}{
Prior works \cite{chen2020optimal,wang2024delta,zhangposter} design client sampling strategies to reduce update variance (Eq. \eqref{eq:varianceRewriteBegin}), our convergence analysis (Theorem \ref{thoerem:convergence}) comprehensively reveals how update variance impacts overall convergence speed. 
However, as Theorem \ref{thoerem:convergence} shows, reducing the term $\mathbb{E}[Z_g^\tau]$ does not necessarily lead to a tighter convergence bound. This is due to the influence of additional terms $\mathbb{E}[Z_l^\tau]$ and $\mathbb{E}[Z_p^\tau]$, which are often overlooked in prior work. 
}


The additional term $\mathbb{E}[Z_l^\tau]$ shows that the relationship between local loss values $f_{i,s}(w_s^\tau)$ and aggregation coefficients $P_{(i,b),s}^\tau$ also influences convergence speed. By applying the triangle inequality, this term can be upper bounded as follows:
\begin{align}
\mathbb{E}\left[(\sum_{i\in\mathcal{N}_s}\sum_{b=1}^{B_i} \mathbbm{1}_{(i,b)}^{s,\tau} P_{(i,b),s}^\tau f_{i,s}(w_s^\tau)-\sum_{i\in\mathcal{N}_s} d_{i,s} f_{i,s}(w_s^\tau))^2\right],
\end{align}
which represents the variance of the surrogate objective:
\begin{align}
\sum_{i\in\mathcal{N}_s}\sum_{b=1}^{B_i} \mathbbm{1}_{(i,b)}^{s,\tau} P_{(i,b),s}^\tau f_{i,s}(w_s^\tau),
\end{align}
implicitly optimized by the clients for model $s$ in each round.

In $\mathbb{E}[Z_p^\tau]$, the first term can be viewed as a constant. The second term reveals that the variance of aggregation coefficients also affects the convergence speed. If sampling probabilities are too heterogeneous across clients, this could lead to unstable training despite equal client contributions. 
\haoran{The impact of $\mathbb{E}[Z_p^\tau]$ will be further explained in Section \ref{sec:CompareGVR-LVR}. 
}

\section{Variance-Reduced Client Sampling}
\label{sec:variance-reduce}
\textcolor{olive}{
Existing gradient-based variance-reduced client sampling algorithms \cite{chen2020optimal, wang2024delta, zhangposter} become impractical in MMFL settings \textcolor{orange}{due to the extra gradient computations among all clients}. Moreover, under highly heterogeneous training conditions, these methods may suffer from instability, as the influence of other critical but unoptimized terms---such as $\mathbb{E}[Z_l^\tau]$ and $\mathbb{E}[Z_p^\tau]$---becomes more pronounced.
In this section, we introduce a \textbf{loss-based variance-reduced client sampling} algorithm that significantly reduces computational overhead while achieving comparable empirical performance to gradient-based approaches (see Section~\ref{sec:exp}).  
}

\subsection{Loss-Based Optimal Variance-Reduced Sampling}


In MMFL systems, gradient-based client sampling methods \cite{chen2020optimal,wang2024delta,zhangposter} require knowledge of the gradient norms $\|G_{(i,b),s}^\tau\|$. Consequently, to compute these values, each client $i$ must perform local training on all models $s$, which is impractical in MMFL settings.  
\haoran{
We propose to instead optimize the term $\mathbb{E}[Z_l^\tau]$ from Theorem \ref{thoerem:convergence}'s convergence upper bound during each training round.}
We solve:
\begin{align}
\label{eq:ASproblem}
&\min_{\mathbf{p}^\tau} \; \sum_{s=1}^{S} \mathbb{E}[(\sum_{(i,b)\in\mathcal{A}_{\tau,s}} P_{(i,b),s}^\tau f_{i,s}(w_s^\tau) -\sum_{i\in\mathcal{N}_s} d_{i,s} f_{i,s}(w_s^\tau))^2] \\
&\text{s.t.}\;  p_{s|(i,b)}^\tau > 0,\; \sum_{s=1}^S p_{s|(i,b)}^\tau \leq 1,\; \nonumber \sum_{s=1}^S\sum_{i\in\mathcal{N}_s}\sum_{b=1}^{B_i} p_{s|(i,b)}^\tau = m, \\
&\; \quad \forall i,b,s,\tau. \nonumber
\end{align}
\textcolor{olive}{
The first two constraints in this optimization problem ensure a feasible probability distribution for each processor $(i,b)$. 
The third constraint ensures that the server expects to receive $m$ total local updates from all clients, by ensuring that $m$ training tasks are assigned to the clients on expectation. This constraint thus models the server's communication limitations, e.g., if the server is at an edge base station, base stations have an upper bound on the number of connections\footnote{We can easily extend our formulation to server bandwidth constraints by constraining the expected size of the models sent to the server, and also the client-side communication constraints by modifying the second constraint to include an upper bound on a client’s participation in training, i.e., $\sum_{s=1}^S p_{s|(i,b)}^\tau \leq \eta_i$.} they can maintain. 
In general, allowing more connections (a higher value of $m$) also implies a need for more sophisticated queuing and parallel processing capabilities at the server, especially if there are thousands of clients and processors, as may be the case in edge FL systems~\cite{mcmahan2017communication}. Intuitively, a high value of $m$ will lead to faster convergence but also higher costs.
}
The closed-form solution of this problem can be obtained (see Supplementary Material).
\begin{theorem}[Optimal MMFL-LVR assignment probabilities]
\label{theorem:solutionAS}
Equation (\ref{eq:ASproblem})'s optimization problem is solved by
\begin{align}
p_{s|(i,b)}^\tau =
\begin{cases}
\frac{(m-V+k)\|\tilde{U}_{(i,b),s}^\tau\|}{\sum_{(j,v) \in \mathcal{V}_{0}} M_{(j,v)}^\tau} & \text{if } (i,b) \in \mathcal{V}_{0}, \\
\frac{\|\tilde{U}_{(i,b),s}^\tau\|}{M_{(i,b)}^\tau} & \text{if } (i,b) \notin \mathcal{V}_{0},
\end{cases}
\label{eq:solutionAS}
\end{align}
where $k = |\mathcal{V}_{0}|$, $V=\sum_{i\in\mathcal{N}_1 \cup\cdots \mathcal{N}_S} B_i$, $\tilde{U}_{(i,b),s}^\tau = \frac{d_{i,s}}{B_i} f_{i,s}(w_s^\tau)$, $M_{(i,b)}^\tau = \sum_{s \in \mathcal{S}_i} \|\tilde{U}_{(i,b),s}^\tau\|$. $\mathcal{V}_{0}$ is the largest set satisfying
\begin{align}
0 < (m-V+k) \leq \frac{\sum_{(j,v) \in \mathcal{V}_{0}} M_{(j,v)}^\tau}{\max_{(i,b) \in \mathcal{V}_{0}} [M_{(i,b)}^\tau]}.\nonumber
\end{align}
\end{theorem}
\textcolor{olive}{
By minimizing the variance in Eq. \eqref{eq:ASproblem}, the surrogate objective $\sum_{(i,b)\in\mathcal{A}_{\tau,s}} P_{(i,b),s}^\tau f_{i,s}(w_s^\tau)$ that the system actually optimizes in each round for each model should be closer to the actual global objective (the expectation of the surrogate objective). 
}
Computing Theorem~\ref{theorem:solutionAS}'s solution \textit{only requires all clients to upload their local loss values}. These can be computed with a forward pass of each model, using much less computational resources than the gradient computations needed for prior works \cite{chen2020optimal,zhangposter}.
We name the \textbf{L}oss-based optimal \textbf{V}ariance-\textbf{R}educed sampling algorithm that uses Eq. \eqref{eq:solutionAS} as \textbf{MMFL-LVR}. 
The pseudocode of MMFL-LVR is shown in Algorithm \ref{algo:AS}. 

\textit{Lower bound of the probability:}
Theorem \ref{thoerem:convergence} requires that the probability has a lower bound: $p_{s|(i,b)}^\tau>\theta$. Otherwise, the proof may not ensure convergence as some clients may rarely or never participate in the training. 
To avoid such cases, a small constant can be added to the local loss $\frac{d_{i,s}}{B_i}f_{i,s}(w_s^\tau)$, which does not affect the practical distribution but theoretically ensures convergence.


\begin{algorithm}[t]
	\caption{MMFL-LVR}
 \label{algo:AS}
 \small
	\begin{algorithmic}[1]
        \State \textbf{Input:} expected number of training tasks $m$, the set of available clients for each model: $\mathcal{N}_s$, computation ability $B_i$ for each client
        \For{global round $\tau=1,\dots,T$}
        \State each client $i$ computes loss $f_{i,s}(w_s^\tau), \forall s$ (in parallel)
        \State each client $i$ sends $\frac{d_{i,s}}{B_i}f_{i,s}(w_s^\tau)$ to the server (in parallel)
        \State server generates $\mathbf{p}^\tau$ using Eq. \eqref{eq:solutionAS}
        \State server generates task allocation ($\mathcal{A}_{\tau,s}$) from $\mathbf{p}^\tau$, requests updates
        \EndFor
        \For{client $i\in\{i:(i,b)\in \mathcal{A}_{\tau,s},\forall s\}$, in parallel}
        \For{\textit{processor} $(i,b)\in \mathcal{A}_{\tau,s'},s'\in\mathcal{S}$, in parallel} 
        \State conducts local training to obtain $G_{(i,b),s}^\tau$
        \State sends $G_{(i,b),s}^\tau$ to the server
        \EndFor
        \EndFor
        \State server conducts aggregation using Eq. \eqref{eq:aggregation}
	\end{algorithmic} 
\end{algorithm}

\subsection{Comparison of GVR and LVR}
\label{sec:CompareGVR-LVR}
\textcolor{olive}{
Gradient-based optimal sampling methods~\cite{chen2020optimal} have been extended to the MMFL setting, as shown in~\cite{zhangposter}. Building on these foundations, we adapt this approach to our heterogeneous MMFL scenario by applying similar proof techniques, which we detail in the Supplementary Material. For clarity, we refer to this adapted gradient-based method for heterogeneous MMFL as \textbf{MMFL-GVR}.
}

\textcolor{olive}{
Here, we use the term $\mathbb{E}[Z_p^\tau]$ from Theorem~\ref{thoerem:convergence} to provide a theoretical comparison between MMFL-GVR and our proposed MMFL-LVR, highlighting that well-established gradient-based methods may potentially destabilize training as the MMFL system becomes more heterogeneous, whereas MMFL-LVR exhibits greater robustness under such conditions.}

As discussed in Section \ref{sec:convergence}, the first term of $\mathbb{E}[Z_p^\tau]$ can be treated as a constant, while the second term: 
\begin{align}
\mathbb{E}[(\sum_{i\in \mathcal{N}s}\sum_{b=1}^{B_i} \mathbbm{1}{(i,b)}^{s,\tau}P{(i,b),s}^\tau -1)^2]
\end{align}
captures the effect of \textbf{participation variance}. 
We begin by explaining how participation variance affects training and then use this theoretical understanding to compare MMFL-GVR and MMFL-LVR. 
\haoran{Firstly, we rewrite the aggregation rule defined in Eq. \eqref{eq:aggregation} as: 
\begin{align}
w_s^{\tau+1}=w_s^\tau - H_{\tau,s}^\top G_s
\label{eq:rewriteaggregation}
\end{align}
where $H_{\tau,s}=[\cdots,\mathbbm{1}_{(i,b)}^{s,\tau} P_{(i,b),s}^\tau ,\cdots]^\top $, $G_s=[\cdots, G_{(i,b),s}^\tau, \cdots]^\top$. The expectation of $\|H_{\tau,s}\|_1$ over the sampled processors ($\|\cdot\|_1$ indicates the $\ell_1$ norm) is: 
\begin{align}
\mathbb{E}[\|H_{\tau,s}\|_1]
=\mathbb{E}\left[\sum_{i\in \mathcal{N}_s}\sum_{b=1}^{B_i} \mathbbm{1}_{(i,b)}^{s,\tau} \frac{d_{i,s}}{B_i p_{s|(i,b)}^\tau} \right]
=\sum_{i\in \mathcal{N}_s}\sum_{b=1}^{B_i} \frac{d_{i,s}}{B_i}=1
\label{eq:H_mean}
\end{align}
In full participation, where each processor trains all models per round, $\|H_{\tau,s}\|_1=1$. 
In the case of partial participation with aggregation rule as Eq. \eqref{eq:rewriteaggregation}, the expected value of $\|H_{\tau,s}\|_1$ remains $1$. 
Thus, $\mathbb{E}[(\sum_{i\in \mathcal{N}_s}\sum_{b=1}^{B_i} \mathbbm{1}_{(i,b)}^{s,\tau} P_{(i,b),s}^\tau -1)^2]$ can be interpreted as the variance of $\|H_{\tau,s}\|_1$. 
As shown in Eq. \eqref{eq:rewriteaggregation}, $\|H_{\tau,s}\|_1$ represents the ``global step size.'' 
A higher variance in $\|H_{\tau,s}\|_1$ indicates instability in this ``global step size.'' This instability can disrupt training, especially in the later stages as model parameters approach optimality, potentially causing significant oscillations.
} 
We find in practice that \textit{MMFL-GVR is much more likely to suffer from this instability}: Figure \ref{fig:stepsize} shows the total global step size ($\sum_{s=1}^S \|H_{\tau,s}\|_1$) for MMFL-GVR and MMFL-LVR (the same experiment settings as Section \ref{sec:setup}).
MMFL-GVR's unstable global step size can be attributed to the following:
Compared to MMFL-LVR, the sampling distribution $\mathbf{p}^\tau$ generated by MMFL-GVR is more unbalanced between processors, as the local loss values used by MMFL-LVR are typically bounded, while the gradient norms used in MMFL-GVR can vary significantly across clients. As a result, some MMFL-GVR clients have small sampling probabilities $p_{s|(i,b)}^\tau$ and correspondingly larger $P_{(i,b),s}^\tau$ values contributing to $\|H_{\tau,s}\|_1$, leading to a higher variance of $\|H_{\tau,s}\|_1$. 


\begin{figure}[t]
    \centering
    \includegraphics[width=0.49\linewidth]{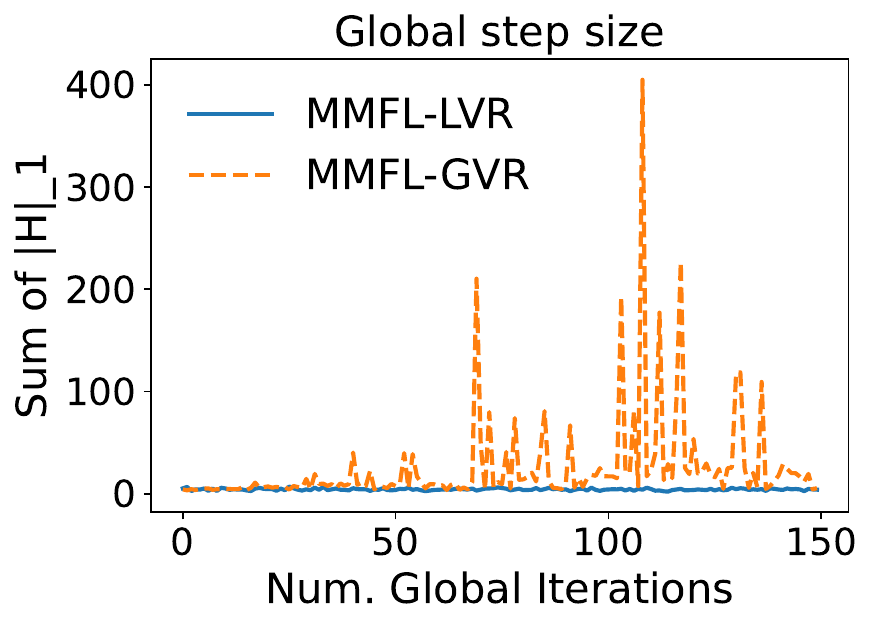}
    \includegraphics[width=0.49\linewidth]{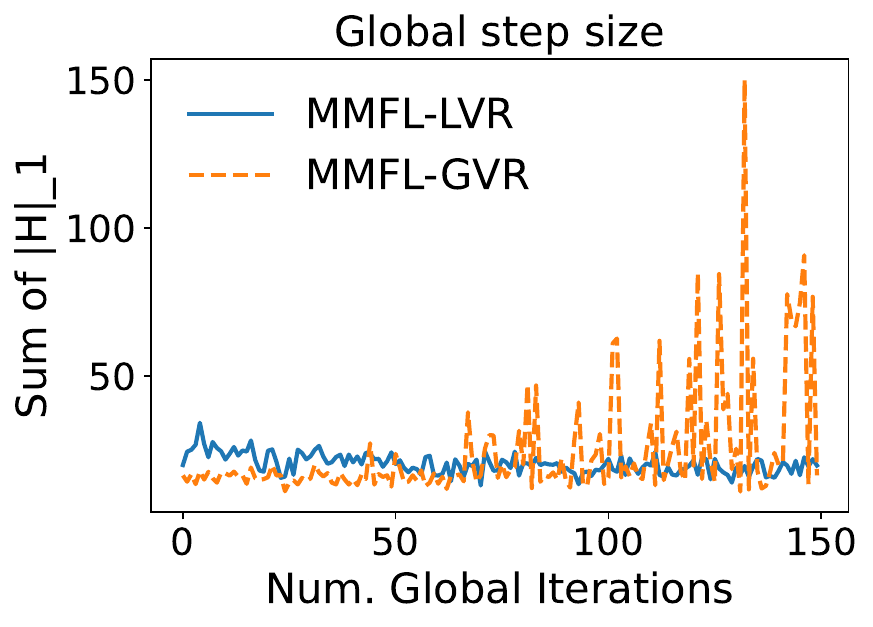}
    \caption{Comparison of the summed global step size of all models ($\sum_{s=1}^S \|H_{\tau,s}\|_1=\sum_{s=1}^S \sum_{(i,b)\in \mathcal{A}_{\tau,s}} P_{(i,b),s}^\tau$) 
    Detailed experiment settings are the same as described in Section \ref{sec:setup}. 
    Left: 3-model setting. Right: 5-model setting. MMFL-GVR's global step size is unstable, potentially harming the training stability. In contrast, MMFL-LVR's participation heterogeneity is much lower, leading to more stable convergence. }
    \label{fig:stepsize}
\end{figure}

\section{\haoran{Optimally Leveraging Stale Information}}

\label{sec:varianceH}
\textcolor{olive}{While MMFL-LVR improves training stability compared to MMFL-GVR, it may still lead to increased participation variance---an issue we address in this section.}
Prior work in SMFL~\cite{jhunjhunwala2022fedvarp, rodio2024fedstale} has shown that leveraging stale client updates can help mitigate the effects of heterogeneous participation and enhance training stability. 
\textcolor{olive}{
Building on this idea, we first extend their stale update aggregation rule from SMFL to our heterogeneous MMFL setting na\"ively, to better understand how incorporating stale information helps reduce the impact of high participation variance. We then propose an improved aggregation scheme that more effectively leverages stale information under such conditions.
The aggregation (na\"ive way) is performed as:}
\begin{align}
w_s^{\tau+1} &= w_s^\tau - \Delta_{\tau,s}\label{eq:stale}\\
\Delta_{\tau,s} &= \beta \sum_{i\in\mathcal{N}_s}  d_{i,s} h_{i,s}^\tau+\sum_{(i,b)\in\mathcal{A}_{\tau,s}}  \frac{d_{i,s} (G_{(i,b),s}^\tau-\beta h_{i,s}^\tau) }{B_i p_{s|(i,b)}^\tau},\nonumber
\end{align}
where $h_{i,s}^\tau$ denotes the last received update as of round $\tau$ from each client $i$ for each model $s$. If client $i$ was active for model $s$ in the previous round $\tau-1$, then $h_{i,s}^\tau=G_{i,s}^{\tau-1}=\mathbb{E}[G_{(i,b),s}^{\tau-1}]$; otherwise, $h_{i,s}^\tau=h_{i,s}^{\tau-1}$.
The expectation of $\Delta_{\tau,s}$ over the processor assignment distribution equals a full participation update, ensuring unbiased training. 
Compared to the aggregation rule in Eq. \eqref{eq:aggregation}, 
Eq. \eqref{eq:stale} adjusts the term $\frac{d_{i,s} G_{(i,b),s}^\tau}{B_i p_{s|(i,b)}^\tau}$, which can dominate the update direction if $p_{s|(i,b)}^\tau$ is small, to $\frac{d_{i,s} (G_{(i,b),s}^\tau - h_{i,s}^\tau)}{B_i p_{s|(i,b)}^\tau}$. 
Since $G_{(i,b),s}^\tau$ and $h_{i,s}^\tau$ often have similar directions,
the scale of $G_{(i,b),s}^\tau - h_{i,s}^\tau$ is generally smaller than that of $G_{(i,b),s}^\tau$ \cite{gu2021fast}, implicitly reducing $\mathbb{E}[Z_p^\tau]$ in Theorem \ref{thoerem:convergence}'s convergence bound.


\begin{figure}[t]
    \centering
    \begin{subfigure}[t]{0.45\linewidth}
        \includegraphics[width=\linewidth]{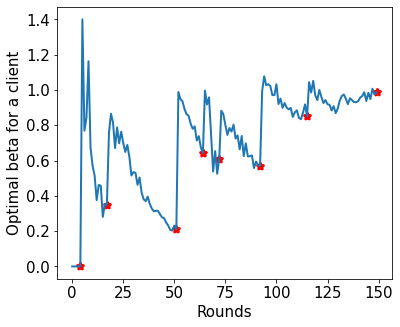}
        \caption{Client 1}
    \end{subfigure}
    \hfill
    \begin{subfigure}[t]{0.45\linewidth}
        \includegraphics[width=\linewidth]{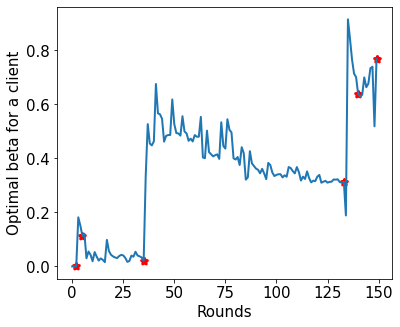}
        \caption{Client 2}
    \end{subfigure}
    \caption{Optimal \( \beta_i^t \) for 2 clients across training rounds. In each subplot, the blue curve represents the optimal \( \beta_i^t \) computed using Eq.~\eqref{eq:solutionStaleb}, while the red stars indicate the rounds during which the client was active.
    \textcolor{olive}{Experiment setting: EMNIST (\(S = 1\), the same setting as Section \ref{sec:fixedDistribution}).}
    }
    \label{fig:betatrend}
\end{figure}

\textcolor{olive}{
However, directly adopting the aggregation rule in Eq.~\eqref{eq:stale}, which balances stale and fresh updates using a single global parameter $\beta$, presents two key limitations. First, finding an optimal global $\beta$ is challenging in practice \cite{jhunjhunwala2022fedvarp}. Second, applying the same $\beta$ across all clients and training rounds fails to account for the varying degrees of staleness among clients, leading to suboptimal performance.}
We therefore propose an \textbf{adaptive aggregation rule} that assigns a dynamically optimized weight $\beta^\tau_{(i,b),s}$ to the stale update from client $i$ at round $\tau$ for model $s$, allowing the aggregation to better reflect client-specific staleness and thereby improving training stability: 
\begin{align}
w_s^{\tau+1} &= w_s^\tau - \Delta_{\tau,s}\label{eq:stale2}\\
\Delta_{\tau,s} &=  \sum_{i\in\mathcal{N}_s} \sum_{b=1}^{B_i} \frac{d_{i,s} z_{(i,b),s}^\tau }{B_i}
+\sum_{(i,b)\in\mathcal{A}_{\tau,s}}  d_{i,s} \frac{ G_{(i,b),s}^\tau-z_{(i,b),s}^\tau }{B_i p_{s|(i,b)}^\tau}\nonumber\\
z_{(i,b),s}^\tau&=\beta_{(i,b),s}^\tau h_{i,s}^\tau\nonumber
\end{align}
To obtain the optimal $\{\beta_{(i,b),s}^\tau\}_{s\in \mathcal{S},i\in \mathcal{N}_s,b\in\mathcal{B}_i}$, we minimize the variance of $\Delta_{\tau,s}$ over the client sampling process. 
Formally, we aim to solve: 
\begin{align}
&\min_{
\{\beta_{(i,b),s}^\tau\}} \sum_{s=1}^S \mathbb{E}\left[\| \frac{\Delta_{\tau,s}}{\eta_{\tau,s}}-\sum_{i\in\mathcal{N}_s}\sum_{b=1}^{B_i} \frac{d_{i,s}}{B_i} \frac{G_{(i,b),s}^\tau}{\eta_{\tau,s}}\|^2\right]. \label{eq:Staleproblem}
\end{align}
We provide the closed-form solution of this problem (proof provided in Supplementary Material). 
\begin{theorem}[MMFL-StaleVR optimal solution] \label{theorem:solutionStale}
Equation \eqref{eq:Staleproblem}'s optimization problem is solved by setting:
\begin{align}
\beta_{(i,b),s}^\tau&= \frac{(G_{(i,b),s}^\tau)^\top h_{i,s}^\tau}{\|h_{i,s}^\tau\|^2}\label{eq:solutionStaleb}
\end{align}
\end{theorem}
Equations \eqref{eq:solutionStaleb} provides the closed-form solution for determining the staleness coefficients in the improved aggregation rule (Eq. \eqref{eq:stale2}), named \textbf{MMFL-StaleVR}. 
Specifically, Eq. \eqref{eq:solutionStaleb} ensures the minimum error between $G_{(i,b),s}^\tau$ and $h_{i,s}^\tau$, i.e., $\|G_{(i,b),s}^\tau-\beta_{(i,b),s}^\tau h_{i,s}^\tau\|$. 
Compared to FedStale \cite{rodio2024fedstale}, our approach is more practical and effectively accelerates convergence by directly optimizing the training variance and balance clients' diverse staleness levels dynamically to achieve the optimal solution. 

The training process of MMFL-StaleVR follows the steps described in Section \ref{sec:MMFLsteps}, replacing the aggregation rule in Eq. \eqref{eq:aggregation} with Eq. \eqref{eq:stale2}, and using Equations  \eqref{eq:solutionAS} and \eqref{eq:solutionStaleb} to guide client sampling and weight stale updates. 
Similar to \cite{jhunjhunwala2022fedvarp,rodio2024fedstale}, MMFL-StaleVR requires server-side additional memory to record and maintain stale updates $\{h_{i,s}^\tau\}_{s\in\mathcal{S}, i\in\mathcal{N}_s}$.
\haoran{Each client also needs to record its own $h_{i,s}^\tau$ for computing $\beta_{(i,b),s}^\tau$ in each round.}
\rachid{This approach mitigates the variance introduced by partial client participation but incurs high computational overhead, as each client must compute gradients for all tasks to obtain the solution in Eq.~\eqref{eq:solutionStaleb}. To address this limitation, we propose an alternative solution that approximates the optimal value of $\beta_{(i,b),s}^\tau$. This approximation is iteratively refined whenever a client is selected to train a specific model. Consequently, only the sampled clients perform gradient computations, while inactive clients incur no additional computational cost---effectively addressing all three research challenges we aim to solve.}

\textcolor{olive}{\textbf{Efficient Approximation:}
We propose an efficient approximation method to estimate the change in the optimal $\beta_i^\tau$ \textbf{without introducing any additional computational cost for inactive clients}. 
Intuitively, when the stale update \( h_{i,s}^\tau \) becomes outdated, its deviation from the true gradient \( G_{(i,b),s}^\tau \) increases, which often leads to a reduced value of \( \beta_{(i,b),s}^\tau \). 
Figure~\ref{fig:betatrend} illustrates the evolution of the optimal \( \beta_{(i,b),s}^\tau \) (Eq. \eqref{eq:solutionStaleb}) for two randomly selected clients. 
Based on the observed approximately linear decay in the optimal values, we propose the following approximation scheme.
}

\textcolor{olive}{
For ease of explanation, we assume \(S = 1\), \(B_i = 1\; \forall i\), i.e., a standard SMFL setting. Extension to MMFL is straightforward.}
Under this setting, assume client $i$ is active at time steps: $\tau^*_i=\tau^1_i, \tau^2_i, \dots, \tau^T_i$. The server receives the corresponding update $G_i^{\tau_i^*}$, and refreshes the stale update in the next round $h_i^{\tau_i^*+1}=G_i^{\tau_i^*}$. For a time step $\tau_i^m+1$ (with $1<m<T$), we approximate the optimal $\beta_i^{\tau_i^m+1}$ as: 
\(
\hat{\beta}_i^{\tau_i^m+1}=\frac{(G_i^{\tau_i^m+1})^\top G_i^{\tau_i^m}}{\|G_i^{\tau_i^m}\|^2} \approx 1.\label{eq:approx1}
\)
This approximation measures the similarity between gradients in consecutive rounds and is expected to be close to 1.
At the next active round, $\tau_i^{m+1}$, $\beta_i^{\tau_i^{m+1}}$ is computed (Eq. \eqref{eq:solutionStaleb}) without imposing any additional workload on the client. 
Intuitively, as the client's update becomes more outdated, the stale information becomes less reliable, and hence the corresponding $\beta_i^\tau$ should decrease over time.
To estimate $\beta_i^\tau$ for any time $\tau$ in the interval $\tau_i^{m+1}+1<\tau<\tau_i^{m+2}$, we linearly interpolate between $\hat{\beta}_i^{\tau_i^m+1}$ and $\beta_i^{\tau_i^{m+1}}$: 
\begin{align}
\beta_{i}^\tau=\hat{\beta}_i^{\tau_i^m+1}+(\tau-\tau_i^{m+1}-1) \frac{\hat{\beta}_i^{\tau_i^m+1}-\beta_i^{\tau_i^{m+1}}}{\tau_i^m+1-\tau_i^{m+1}}.
\end{align}
This interpolation approximates the future trend of $\beta_i^t$ based on historical observations without additional client-side computation. We refer to this method as \textbf{MMFL-StaleVRE} (\textbf{V}ariance-\textbf{R}educed \textbf{E}stimation of the optimal solution Eq. \eqref{eq:solutionStaleb}).

\section{Experiment and Evaluation}
\label{sec:exp}

\begin{figure}[t]
    \centering
    \includegraphics[width=0.49\linewidth]{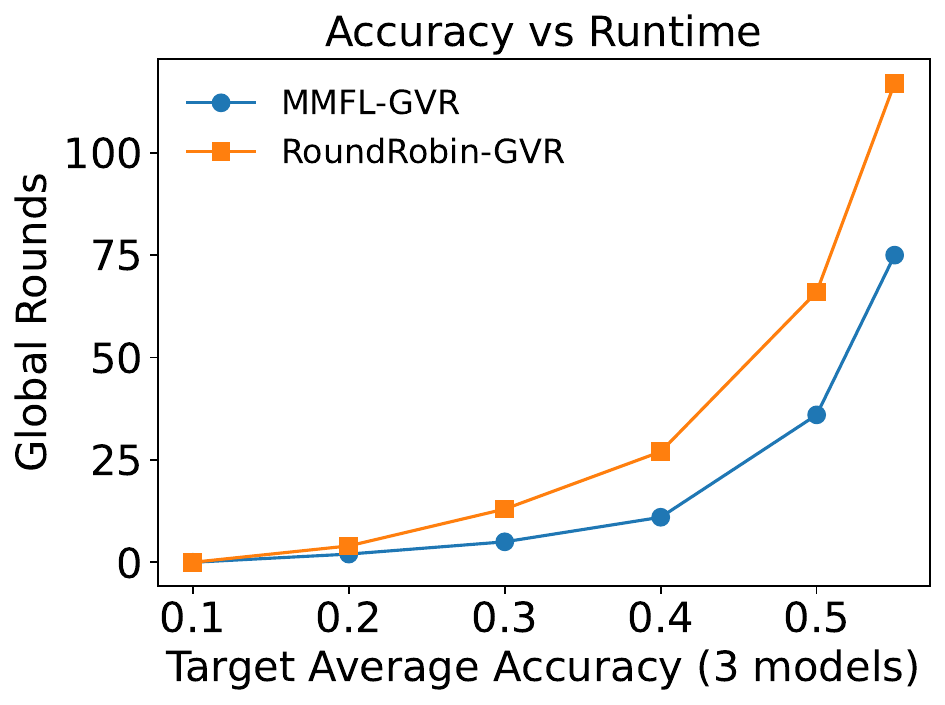}
    \includegraphics[width=0.49\linewidth]{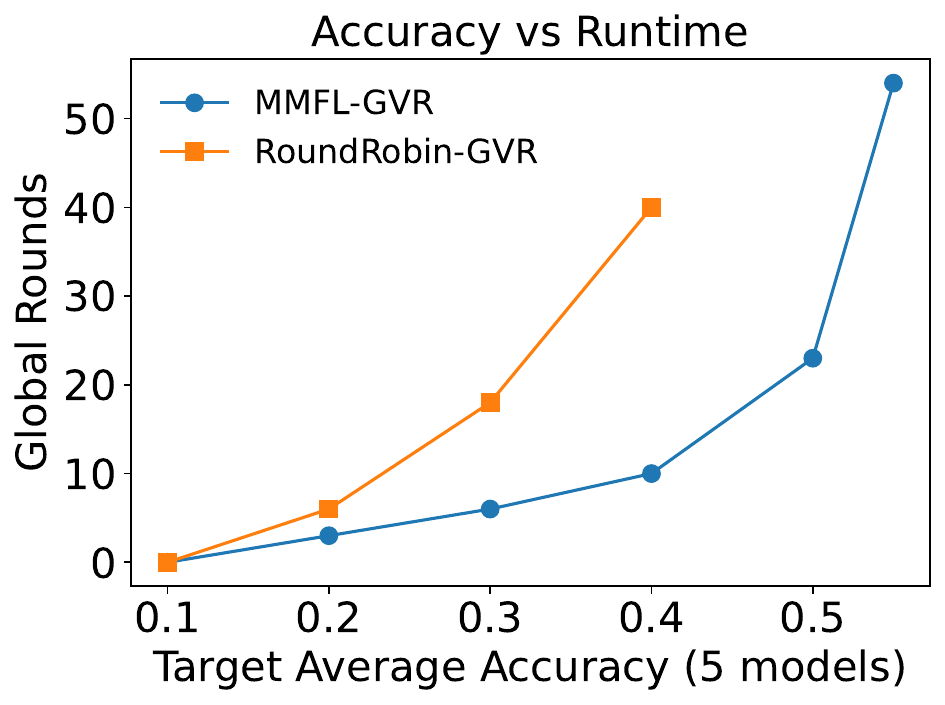}
    \caption{
    Comparison of MMFL-GVR and RoundRobin-GVR: target accuracy vs required global rounds. 
    \textcolor{olive}{
    \textit{Lower curve indicates MMFL algorithm achieves the target accuracy faster than SMFL algorithm.}}
    Left: 3-model setting. Right: 5-model setting. 
    RoundRobin-GVR fails to achieve target accuracies of 0.5 and 0.55 within 150 rounds. 
    }
    \label{fig:SMFL}
\end{figure}


\haoran{In this section, we conduct extensive experiments to demonstrate the advantages of the proposed algorithms (MMFL-LVR, MMFL-StaleVR, and MMFL-StaleVRE).}

We first demonstrate the benefit of allowing concurrent model training in MMFL by comparing it to a round-robin baseline, where models are trained sequentially—one per round in rotation. 
For both MMFL and SMFL (round-robin), we apply gradient-based optimal sampling \cite{chen2020optimal, zhangposter} to determine which clients participate in each round, as it is well-established in the SMFL literature and recently extended to MMFL. 
Specifically, MMFL employs an extension of the method in \cite{zhangposter}, which we adapt to handle heterogeneous client resources (referred to as MMFL-GVR), while the round-robin baseline trains one model at a time using the formulation from \cite{chen2020optimal} (referred to as RoundRobin-GVR). 
The final model accuracy with this baseline is also equivalent to sequentially training each model one at a time for a fixed number of global iterations. 
Thus, comparing to RoundRobin-GVR shows the value of allowing clients to opportunistically train other models.

We then show the value of optimizing the client sampling with our loss-based method (MMFL-LVR) in MMFL by comparing to MMFL-GVR and a \textit{random} baseline, i.e., each client processor is uniformly at random allocated to a model. We also compare our proposed algorithms to extensions of FedVARP \cite{jhunjhunwala2022fedvarp}, MIFA \cite{gu2021fast}, and SCAFFOLD \cite{karimireddy2020scaffold}, where we again uniformly at random allocate client processors to models. These SMFL algorithms all aim to reduce the variance through modifying the local training (SCAFFOLD) or aggregation (FedVARP and MIFA). 
Comparison to these uniform sampling baselines thus shows the value of optimizing the client sampling across models so as to reduce the model variance, while respecting client resource constraints.  We wish to make our algorithms' achieved model accuracy as close as possible to our \textit{full participation} baseline, in which all clients train all models in each training round (thus disregarding client resource constraints). 

Since the proposed algorithms all naturally accommodate heterogeneous client resource constraints, we do not compare their performance to existing MMFL allocation methods that restrict clients to train one model in each round~\cite{bhuyan2022multi, liu2022multi, askin2024fedast,chang2024asynchronous,liu2023multi,atapour2023multi}.

\begin{table}[t]
\centering
\caption{MMFL-StaleVR has the highest final average model accuracy relative to that from full participation. }
\small
\label{tab:results}
\begin{tabular}{@{}c|c|c@{}}
\toprule
\multirow{2}{*}{Methods} & \multirow{2}{*}{{\bf 3 tasks}} & \multirow{2}{*}{{\bf 5 tasks}} \\
                         &                          &                          \\ \midrule
FedVARP\cite{jhunjhunwala2022fedvarp}                  & $0.681\pm0.16$                 & $0.785\pm0.18$                   \\
MIFA\cite{gu2021fast}                     & $0.854\pm0.18$                  & $0.890\pm 0.18$                   \\
SCAFFOLD\cite{karimireddy2020scaffold}                 & $0.783\pm0.13$                   & $0.734\pm0.17$                   \\
MMFL-GVR \cite{zhangposter}                 & $0.886\pm0.15$                   & $0.872\pm0.20$                   \\
Random                   & $0.792\pm0.17$                   & $0.823\pm0.22$                   \\
Full Participation       & $1.000\pm0.13$                    & $1.000\pm0.14$                    \\ 
\midrule
MMFL-LVR                 & $0.896\pm0.15$                   & $0.893\pm0.16$                   \\
MMFL-StaleVR                & $\mathbf{0.943}\pm0.15$                   & $\mathbf{0.946}\pm0.17$ 
\\
MMFL-StaleVRE                & $0.918\pm0.17$                   & $0.908\pm0.19$ 
\end{tabular}
\end{table}

\subsection{Experimental Setup}
\label{sec:setup}

\textbf{Datasets.} We use Fashion-MNIST, EMNIST, CIFAR-10, and Shakespeare \cite{mcmahan2017communication}. 
We have 120 clients in total. 
For the first three datasets, each client receives data from 30\% of the labels. \textcolor{brown}{To further increase data heterogeneity across multiple models, clients are divided into two groups for each model: high-data clients (comprising 10\% of the total clients, each holding around 120 datapoints for the respective model) and low-data clients (the remaining 90\%, each holding around 12 datapoints for the respective model). Importantly, this division is model-specific, meaning a client can be a high-data client for one model and a low-data client for another model. 
Consequently, 10\% of the clients hold approximately 52.6\% of the total data for each model.}
The Shakespeare dataset is naturally non-iid, so we select 120 clients \textcolor{brown}{uniformly} from its total 1146 clients each corresponding to a Shakespeare character) without modifications. 

\textbf{Model and training configuration.} \textcolor{black}{The number of local training epochs is set to $E=5$ for all models. 
For the Fashion-MNIST and EMNIST classification tasks, we construct similar convolutional neural networks (CNN), each with 2 convolutional layers, 2 pooling layers, and 2 linear layers, and with different output layer sizes. 
For the CIFAR-10 task, we use a pre-activation ResNet \cite{he2016resnet}. For the Shakespeare dataset, we implement a character-level LSTM language model with an embedding layer, a two-layer LSTM, and a linear layer. 
All algorithms are implemented in Pytorch 2.2.1 with \texttt{SGD optimizer} \cite{paszke2017pytorch}. 
Experiments are performed and results are averaged over \textbf{5 random seeds}.}

\textbf{Client resource heterogeneity.}
We have $S$ models, and assume 90\% of clients can train all $S$ models, while 10\% can train $S-1$ models (randomly decided). Each client $i$ has $B_i$ processors, allowing them to train $B_i$ models in parallel per round. Clients' $B_i$ distribution is as follows: 25\% of clients have $B_i=|\mathcal{S}_i|$ (the number of models available for client $i$), 50\% have $B_i=\lceil \frac{|\mathcal{S}_i|}{2}\rceil$, and 25\% have $B_i=1$.
Based on this setting, we have $V=\sum_{i=1}^N B_i$ \textit{processors}. We set the active rate as 10\%, i.e., $m=\frac{1}{10} V$ for the proposed client sampling algorithms. And we assume each \textit{processor} is active with 10\% probability for other baselines.

\begin{figure}[t]
    \centering
    \includegraphics[width=0.75\linewidth]{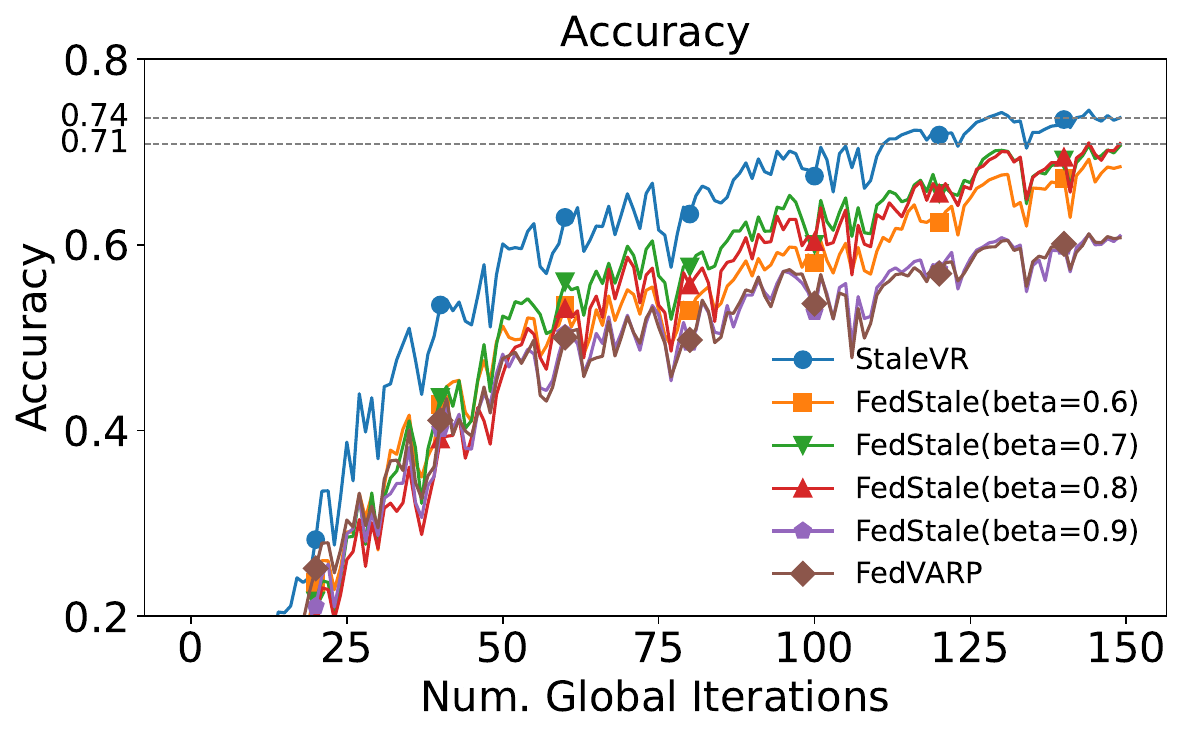}
    \caption{Effect of staleness weights with fixed sampling distribution. We evaluate the effect of our dynamic staleness weights (Eq. \eqref{eq:solutionStaleb}) with a fixed client sampling distribution using EMNIST ($S=1$). Clients are divided into two groups with participation rates of 4\% and 16\%. MMFL-StaleVR achieves a final accuracy of 0.74, outperforming FedStale and FedVARP, which use static weights for the stale updates (max 0.71). }
    \label{fig:beta}
\end{figure}

\subsection{Experiment Results}
\subsubsection{Comparsion of MMFL to na\"ive extension of SMFL}
Figure \ref{fig:SMFL} compares MMFL-GVR with the RoundRobin-GVR in two settings: 3-model (left graph) and 5-model (right graph). 
For 3-model setting, we include three Fashion-MNIST models. For 5-model setting, we include two Fashion-MNIST models, one CIFAR-10 model, one EMNIST model, and one Shakespeare model. 
MMFL-GVR consistently requires fewer global rounds to reach a given accuracy compared to RoundRobin-GVR. As the target accuracy increases, the gap between the two methods widens. In the 5-model setting, RoundRobin-GVR fails to achieve the target accuracies of 0.5 and 0.55 within 150 rounds, further highlighting its inefficiency in managing convergence for multiple models. 
Although RoundRobin-GVR's optimal sampling of all clients for one model in each round (i.e., MMFL-GVR with $S = 1$) reduces update variance, 
this restriction prevents clients from opportunistically training other models to which they can better contribute. It thus slows down the overall training process for all models, especially in the early stages of training.

\subsubsection{Comparison to algorithms with uniform sampling}
We conducted extensive experiments to compare our proposed MMFL algorithms with multiple baselines, following the same 3-model and 5-model settings described above. 
The quantitative results for both settings are summarized in Table \ref{tab:results}, with the corresponding accuracy trends provided in the Supplementary Material. 
We report the relative accuracy compared to that of full participation (i.e., each method's average final accuracy divided by the average final accuracy under full participation). 
All proposed MMFL algorithms outperform the baselines, with MMFL-StaleVR achieving the best results by effectively optimizing both the sampling distribution and staleness coefficients. Specifically, MMFL-StaleVR delivers up to 19.1\% higher average accuracy compared to random allocation and narrows the gap to full participation---the theoretical best performance in this experiment setting---to just 5.4\%. 
MMFL-LVR outperforms MMFL-GVR, despite requiring much less local training.
\textcolor{olive}{MMFL-StaleVRE further improves MMFL-LVR without any additional computation workload for inactive clients.}
In the 5-model experiment, since we use the same participation rate as in the 3-model setting, each model's training resource is more limited when spread over five models, leading to a performance drop for all methods compared to \textit{full participation}. 

\subsubsection{Effectiveness of optimally leveraging stale updates}
\label{sec:fixedDistribution}
We further evaluate the contribution of the proposed dynamic staleness weights (Eq. \eqref{eq:solutionStaleb}) in MMFL-StaleVR by fixing the client sampling distribution. The experiment was conducted using only EMNIST ($S=1$). To simulate participation heterogeneity, we divide total 40 clients into two equal groups; the first group has a participation rate of 4\%, while the other group has a participation rate of 16\%. In this setting, we compare MMFL-StaleVR with FedStale and FedVARP. 
The results (Fig. \ref{fig:beta}) show that MMFL-StaleVR achieves a final accuracy of 0.74, outperforming FedStale and FedVARP, which achieves a maximum accuracy of 0.71 among different values of $\beta$. Unlike FedStale, where the optimal staleness coefficient cannot be computed in practice, MMFL-StaleVR provides a practical closed-form solution (Eq. \eqref{eq:solutionStaleb}) to optimally weigh stale updates for clients with diverse staleness levels.

\section{Conclusion}\label{sec:conclusion}
In this work, we study an MMFL system with heterogeneous client resources and model availability. We begin by analyzing convergence under arbitrary client sampling strategies, reaffirming the benefits of gradient-based methods established in prior work, while also highlighting their limitations---including high computational and communication overhead, and instability due to resulted sampling distributions.
To address these challenges, we propose a suite of algorithms: MMFL-LVR reduces computational and communication costs through loss-based sampling; MMFL-StaleVR enhances training stability by optimally incorporating stale client updates; and MMFL-StaleVRE approximates the optimal use of stale updates with minimal overhead, collectively addressing all three research challenges.
\textcolor{olive}{In future work, the system can be extended to model more fine-grained and realistic communication constraints at the per-client level.}



\bibliographystyle{ACM-Reference-Format}
\bibliography{references}

\onecolumn

\section{Supplementary Material}
\label{sec:app}
\subsection{Additional Results}

\begin{figure}[h]
    \centering
\includegraphics[height=0.25\linewidth]{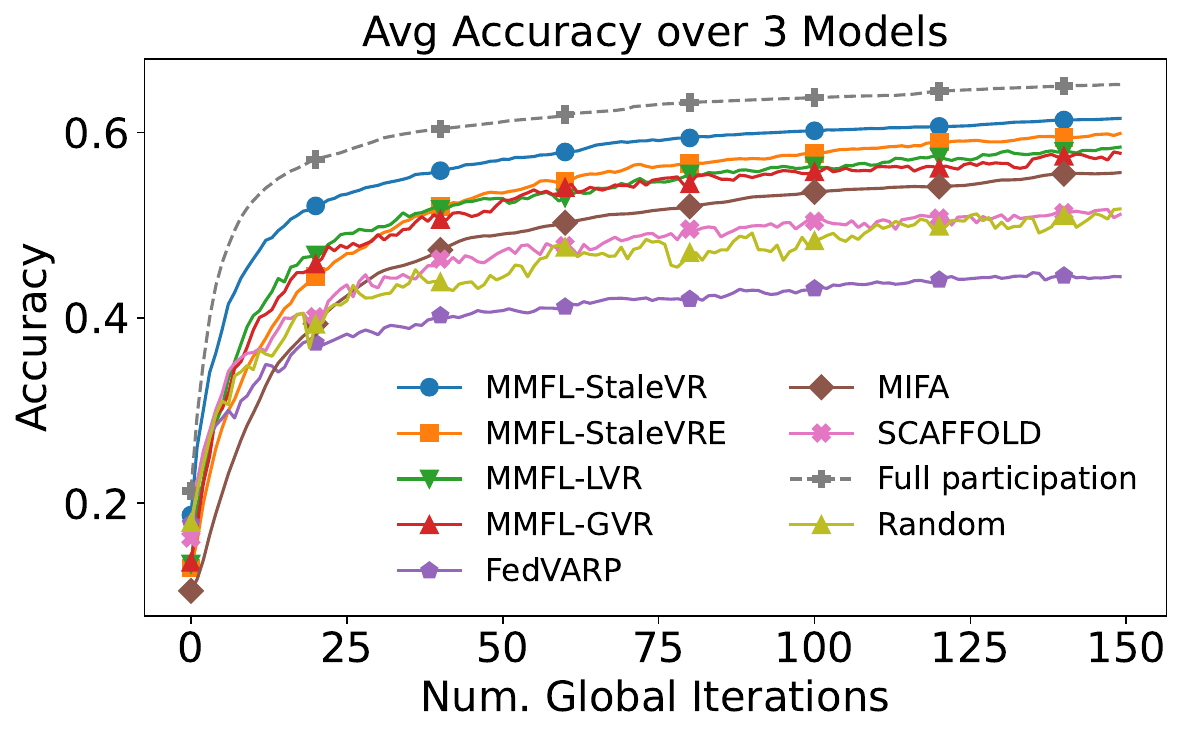}
\includegraphics[height=0.25\linewidth]{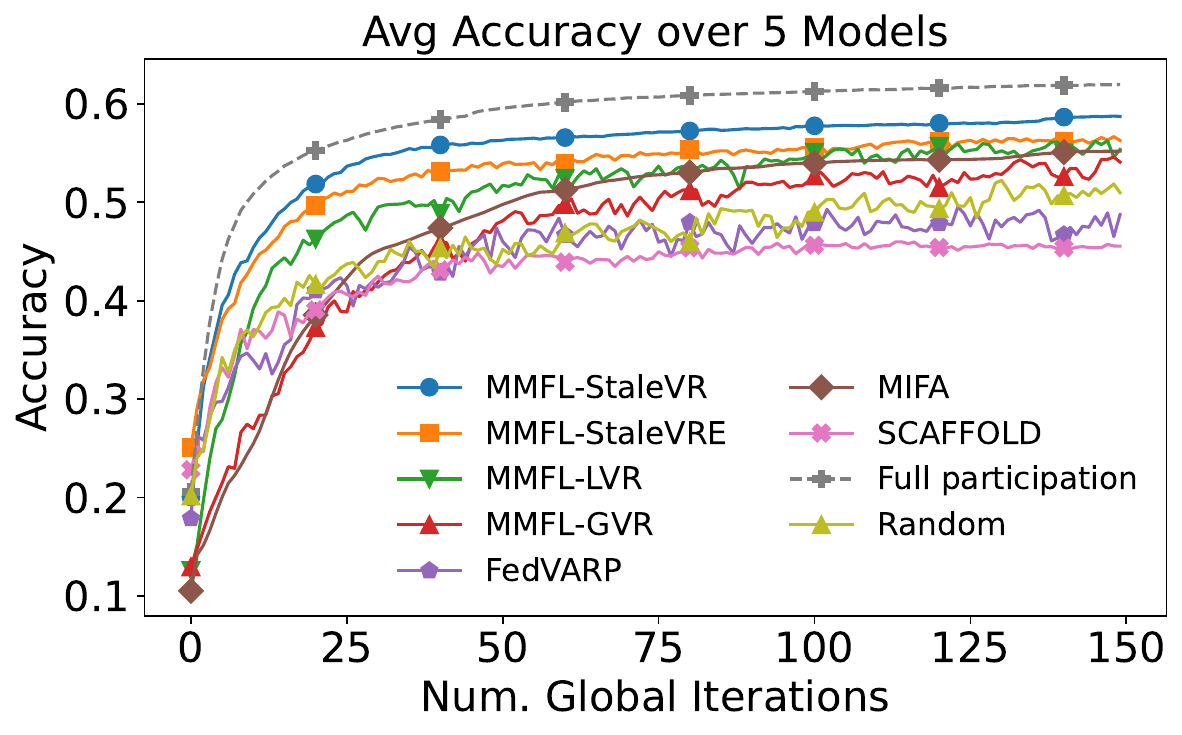}
    \caption{Average accuracy of different methods.  }
    \label{fig:avgacc}
\end{figure}

\begin{table}[H]
\centering
\caption{System overheads for deploying different algorithms. }
\label{tab:cost}
\begin{tabular}{@{}c|ccc@{}}
\toprule
\multirow{2}{*}{Methods} & \multirow{2}{*}{\begin{tabular}[c]{@{}c@{}}Comm. \\ Cost\end{tabular}} & \multirow{2}{*}{\begin{tabular}[c]{@{}c@{}}Comp. \\ Cost\end{tabular}} & \multirow{2}{*}{\begin{tabular}[c]{@{}c@{}}Mem. \\ Cost\end{tabular}} \\
                         &                                                                        &                                                                        &                                                                       \\ \midrule
Full Participation       & $CSN$                                                                  & $TSN$                                                                   & $(N+1)SM$                                                                   \\ \midrule
MMFL-GVR                 & $(qC+1)N$                                                              & $TSN$                                                                   & $(N+1)SM$                                                                   \\
MMFL-LVR                 & $(qC+1)N$                                                              & $TqN$                                                                    & $(N+1)SM$                                                                   \\
MMFL-StaleVR             & $(qC+2)N$                                                              & $TSN$                                                                   & $(3N+1)SM$                                                                   \\
MMFL-StaleVRE             & $(qC+2)N$                                                              & $TqN$                                                                   & $(3N+1)SM$                                                                   \\
\bottomrule
\end{tabular}
\end{table}

\subsection{Proofs}
\begin{theorem}[Convergence] Let $w_s^*$ denote the optimal weights of model $s$. If the learning rate $\eta_{\tau,s}=\frac{16}{\mu } \frac{1}{(\tau+1)\rachid{K}+\gamma_{\tau,s}}$, then
\begin{align}
\mathbb{E}\left(\|w_{s}^{\tau} - w_s^*\|^2\right)\leq \frac{V_\tau}{(\tau K +\gamma_{\tau,s})^2} \label{eq:convergence}
\end{align}
We define
$\gamma_{\tau,s}=\max \{\frac{32L}{\mu},4\rachid{K}\sum_{i\in \mathcal{N}_s}\sum_{b=1}^{B_i} \mathbbm{1}_{(i,b)}^{s,\tau} P_{(i,b),s}^\tau\}$, \\
$V_\tau=\max\{\gamma_\tau^2 \mathbb{E}(\|w_s^0-w_s^*\|^2), (\frac{16}{\mu})^2\sum_{\tau'=0}^{\tau-1}z_{\tau'}\}$, \\
$\rachid{z}_{\tau'}=\mathbb{E}[Z_g^{\tau'}+Z_l^{\tau'}+Z_p^{\tau'}]$, \\
\(
\mathbb{E}[Z_g^\tau]=K\sum_{i\in \mathcal{N}_s}\sum_{b=1}^{B_i} \frac{(\frac{d_{i,s}}{B_i}\sigma_{i,s})^2}{p_{s|(i,b)}^\tau}+4LK\sum_{i\in \mathcal{N}_s} d_{i,s} \Gamma_{i,s}+\max(\frac{B_i}{d_{i,s}}) \mathbb{E}[\sum_{i\in \mathcal{N}_s}\sum_{b=1}^{B_i} \frac{(\frac{d_{i,s}}{B_i})^2 \sum_{t=1}^{K} \| \nabla {f}_{i,s} (w_{(i,b),s}^{t,\tau}) \|^2}{p_{s|(i,b)}^\tau}]\),\\
\(
\mathbb{E}[Z_l^\tau]=R \mathbb{E}[V_s  \sum_{i\in \mathcal{N}_s}\sum_{b=1}^{B_i} (\mathbbm{1}_{(i,b)}^{s,\tau}P_{(i,b),s}^\tau f_{i,s}(w_s^\tau)-\frac{d_{i,s}}{B_i}f_{i,s}(w_s^\tau))^2]\), where $R=\frac{2K^3\bar{\sigma}^2}{e_w^2e_f^2 \theta}, V_s=\sum_{i\in\mathcal{N}_s}B_i$, \\
\(
\mathbb{E}[Z_p^\tau]=(\frac{2}{\theta}+K(2+\frac{\mu}{2L})) K^2\bar{\sigma}^2 
+\frac{2\rachid{K}^3\bar{\sigma}^2}{\theta} \mathbb{E}[(\sum_{i\in \mathcal{N}_s}\sum_{b=1}^{B_i} \mathbbm{1}_{(i,b)}^{s,\tau}P_{(i,b),s}^\tau -1)^2]
\).
\label{thoerem:convergence}
\end{theorem}

\textbf{Detailed steps of the proof: }
\textcolor{red}{In total, we have $N$ clients and $S$ models. Client $i$ has $B_i$ processors, each processor can be assigned a training task per round. Total number of processors: $V=\sum_{i=1}^N B_i$. We keep the same FL objective as previous MMFL papers: $F=\sum_{s=1}^S \sum_{i=1}^N d_{i,s} f_{i,s}(w_s)$. 
}

\textbf{0. Equivalent View}

We directly model the sampling process in the processor-level. The processor index $(i,b)$ can be reindexed as $v=1,2\dots,V$ across all clients. 

Here we define $\mathbbm{1}_{(i,b)\in \mathcal{A}_{\tau,s}}$ \textcolor{red}{or $\mathbbm{1}_{v\in \mathcal{A}_{\tau,s}}$ } as the indicator whether \textcolor{red}{processor $(i,b)$} is sampled to model $s$ in round $\tau$. 
For each processor, the local update rule is 
\begin{align}
w_{(i,b),s,\tau}^{t+1}=w_{(i,b),s,\tau}^t-\eta_{\tau,s} \nabla f_{i,s}(w_{(i,b),s,\tau}^t)
\end{align}
Define the change in local weights after $E$ local epochs as 
\begin{align}
G_{(i,b),s}^\tau=\eta_{\tau,s} \sum_{t=0}^{E-1} \nabla f_{i,s}(w_{(i,b),s,\tau}^t)    
\end{align}
Define $P_{(i,b),s}^\tau = P_{v,s}^\tau= \frac{d_{i,s}}{B_i p_{s|(i,b)}^\tau (\pi, w_s^{\tau})} \mathbbm{1}_{(i,b) \in \mathcal{A}_{\tau,s}}$. The global aggregation rule is
\begin{align}
w^{\tau+1}_s&=w^\tau_s - \sum_{i=1}^N \sum_{b=1}^{B_i} P_{(i,b),s}^\tau G_{(i,b),s}^\tau
\label{updaterule1}\\
&=w^\tau_s - \sum_{i=1}^N d_{i,s} \sum_{b=1}^{B_i} \frac{1}{B_i p_{s|(i,b)}^\tau (\pi, w_s^{\tau})}G_{(i,b),s}^\tau
\end{align}
\textcolor{red}{Since each processor's sampling process is independent, multiple processors within the same client could be assigned to the same training task. In such condition, multi-processor MMFL system can be viewed as a two-layer FL system. The first layer is the aggregation from processors to the client, and the second layer is the aggregation from clients to the server. }

For simplification, we use $p_{s|(i,b)}^\tau=p_{s|(i,b)}^\tau (\pi, w_s^{\tau})$ in the proofs below. 
Now we generalize the global weights $w_s^\tau$ to any local epochs. We define $\overline{w}_s^{t}$ such that $\overline{w}_s^0=w_s^0$, and 
\begin{align}
\overline{w}_s^{\tau E+t+1}&=\overline{w}_s^{\tau E+t}-\eta_{\tau,s} \sum_{i=1}^N \sum_{b=1}^{B_i} P_{(i,b),s}^{\tau} \nabla f_{i,s}(w_{(i,b),s,\tau}^t)\\
&=\overline{w}_s^{\tau E+t}-\eta_{\tau,s} \sum_{v=1}^V P_{v,s}^{\tau} \nabla f_{i,s}(w_{v,s,\tau}^t)
\label{updaterule2}
\end{align}
Note that though Eq. \ref{updaterule1} and Eq. \ref{updaterule2} use the sum of the updates from all processors of each client, the aggregation coefficient $P_{v,s}^\tau$ contains the indicator $\mathbbm{1}_{v\in \mathcal{A}_{\tau,s}}$, which reflects if this processor is available to model $s$. 

\begin{lemma}
The relationship between global weights and generalized global weights: for any $\tau$, $\overline{w}_s^{\tau E}=w_s^\tau$. 
\end{lemma}
\begin{proof}
By definition, $\overline{w}_s^0=w_s^0$, suppose $\overline{w}_{s}^{\tau E}=w_s^\tau$, then
\begin{align}
\overline{w}_s^{(\tau+1) E}&=w_{s}^{\tau E +E-1} -\eta_{\tau,s} \sum_{v=1}^V P_{v,s}^\tau \nabla f_{i,s}(w_{v,s,\tau}^{E-1})\\
&=\cdots=\overline{w}_s^{\tau E}-\sum_{t=0}^{E-1}\eta_{\tau,s} \sum_{v=1}^V P_{v,s}^\tau \nabla f_{i,s}(w_{v,s,\tau}^t)\\
&=w_s^\tau -  \sum_{v=1}^V P_{v,s}^\tau \sum_{t=0}^{E-1}\eta_{\tau,s} \nabla f_{i,s}(w_{v,s,\tau}^t)\\
&=w_{s}^{\tau+1}
\end{align}
\end{proof}
In the following analysis, we will only use $\overline{w}_s^{\tau E+t}$ to denote the global weights. 

\begin{lemma}
The expected weighted discrepancy between mini-batch gradient and full batch gradient: 
\begin{align}
\mathbb{E}_{\xi} \left[\|\sum_{v=1}^V P_{v,s}^\tau (\nabla f_{i,s}(w_{v,s,\tau}^t)-\nabla \overline{f}_{i,s}(w_{v,s,\tau}^t))\|^2 \right]
\leq \sum_{v=1}^V (P_{v,s}^\tau)^2 \sigma_i^2
\end{align}
$\xi$ is the mini-batch of data samples we select in round $\tau$.
\label{lemma2}
\end{lemma}
\begin{proof}
For simplification, define $g_{v,s,\tau}^t=\nabla f_{i,s}(w_{v,s,\tau}^t)$. 
\begin{align}
\|\sum_{v=1}^V P_{v,s}^\tau (g_{v,s,\tau}^t- \overline{g}_{v,s,\tau}^t)\|^2
&= \sum_{v=1}^V\| P_{v,s}^\tau (g_{v,s,\tau}^t-\overline{g}_{v,s,\tau}^t)\|^2
+ \sum_{v\neq j} P_{v,s}^\tau P_{j,s}^\tau \langle g_{v,s,\tau}^t-\overline{g}_{v,s,\tau}^t,g_{j,s,\tau}^t-\overline{g}_{j,s,\tau}^t\rangle
\end{align}
Since $g_{v,s,\tau}^t$ and $g_{j,s,\tau}^t$ are independent, the covariance should be 0.  
\begin{align}
\mathbb{E}_\xi \langle g_{v,s,\tau}^t-\overline{g}_{v,s,\tau}^t,g_{j,s,\tau}^t-\overline{g}_{j,s,\tau}^t\rangle = 0
\end{align}
Thus, 
\begin{align}
\mathbb{E}_\xi \left[\|\sum_{v=1}^V P_{v,s}^\tau (g_{v,s,\tau}^t- \overline{g}_{v,s,\tau}^t)\|^2\right]
&=\mathbb{E}_\xi \left[ \sum_{v=1}^V\| P_{v,s}^\tau (g_{v,s,\tau}^t- \overline{g}_{v,s,\tau}^t)\|^2\right]\\
&=\sum_{v=1}^V (P_{v,s}^\tau)^2 \mathbb{E}_\xi \left[ \| (g_{v,s,\tau}^t- \overline{g}_{v,s,\tau}^t)\|^2\right]\\
&\leq \sum_{v=1}^V (P_{v,s}^\tau)^2 \sigma_i^2
\end{align}
\end{proof}

\begin{lemma} Weighted discrepancy between global model parameters $\overline{w}_s^{\tau E+t}$ and processor's local model parameters $w_{v,s,\tau}^t$: 
for all $t=0,\cdots,E-1$ and all $\tau$
\begin{align}
\mathbb{E}_\xi \left[
\sum_{v=1}^V P_{v,s}^\tau \| \overline{w}_s^{\tau E +t} - w_{v,s,\tau}^t\|^2 \right] \leq 
\eta_{\tau,s}^2 E^2 G^2 \left(
\theta \left(\sum_{v=1}^V P_{v,s}^\tau -1\right)^2 
+
\sum_{v=1}^V P_{v,s}^\tau
\right)
\end{align}
\label{lemma3}
\end{lemma}
\begin{proof}
\begin{align}
\sum_{v=1}^V P_{v,s}^\tau \|\overline{w}_{s}^{\tau E+t} -w_{v,s,\tau}^t\|^2
&=\sum_{v=1}^V P_{v,s}^\tau
\|\overline{w}_{s}^{\tau E+t} -w_{v,s,\tau}^t -\overline{w}_s^{\tau E}+\overline{w}_s^{\tau E}\|^2\\
&=\sum_{v=1}^V P_{v,s}^\tau \left(\|\overline{w}_{s}^{\tau E+t} -\overline{w}_s^{\tau E}\|^2
-2\langle\overline{w}_{s}^{\tau E+t} -\overline{w}_s^{\tau E}, w_{v,s,\tau}^t- \overline{w}_s^{\tau E} \rangle+ \| w_{v,s,\tau}^t- \overline{w}_s^{\tau E} \|^2\right)
\end{align}
Consider
\begin{align}
\sum_{v=1}^V P_{v,s}^\tau w_{v,s,\tau}^t
&=\sum_{v=1}^V P_{v,s}^\tau \left(w_{v,s,\tau}^{t-1} - \eta_{\tau,s} g_{v,s,\tau}^{t-1}\right)\\
&=\sum_{v=1}^V P_{v,s}^\tau w_{v,s,\tau}^{t-1} - \eta_{\tau,s} \sum_{v=1}^V P_{v,s}^\tau g_{v,s,\tau}^{t-1}\\
&=\sum_{v=1}^V P_{v,s}^\tau w_{v,s,\tau}^{t-1} +(\overline{w}_{s}^{\tau E +t}- \overline{w}_{s}^{\tau E +t-1})\\
&=\cdots =\sum_{v=1}^V P_{v,s}^\tau w_{v,s,\tau}^{0} +(\overline{w}_{s}^{\tau E +t}- \overline{w}_{s}^{\tau E}) \label{eq:l1}\\
&=\sum_{v=1}^V P_{v,s}^\tau w_s^{\tau E} +(\overline{w}_{s}^{\tau E +t}- \overline{w}_{s}^{\tau E}) \label{eq:l2}
\end{align}
Use Eq. \ref{eq:l2} to simplify the cross-term: 
\begin{align}
-2\sum_{v=1}^V P_{v,s}^\tau \langle\overline{w}_{s}^{\tau E+t} -\overline{w}_s^{\tau E}, w_{v,s,\tau}^t- \overline{w}_s^{\tau E} \rangle
&=-2 \langle\overline{w}_{s}^{\tau E+t} -\overline{w}_s^{\tau E}, \sum_{v=1}^V P_{v,s}^\tau w_{v,s,\tau}^t-\sum_{v=1}^V P_{v,s}^\tau \overline{w}_s^{\tau E} \rangle\\
&=-2 \langle\overline{w}_{s}^{\tau E+t} -\overline{w}_s^{\tau E}, \sum_{v=1}^V P_{v,s}^\tau w_s^{\tau E} +(\overline{w}_{s}^{\tau E +t}- \overline{w}_{s}^{\tau E})-\sum_{v=1}^V P_{v,s}^\tau \overline{w}_s^{\tau E} \rangle\\
&=-2 \langle\overline{w}_{s}^{\tau E+t} -\overline{w}_s^{\tau E},\overline{w}_{s}^{\tau E +t}- \overline{w}_{s}^{\tau E} \rangle\\
&=-2\|\overline{w}_{s}^{\tau E+t} -\overline{w}_s^{\tau E}\|^2
\end{align}
Thus
\begin{align}
\sum_{v=1}^V P_{v,s}^\tau \|\overline{w}_{s}^{\tau E+t} -w_{v,s,\tau}^t\|^2
&=(\sum_{v=1}^V P_{v,s}^\tau -2)\|\overline{w}_{s}^{\tau E+t} -\overline{w}_s^{\tau E}\|^2+\sum_{v=1}^V P_{v,s}^\tau \| w_{v,s,\tau}^t- \overline{w}_s^{\tau E} \|^2 \label{eq:l3}
\end{align}
Bounding the first term of Eq. \ref{eq:l3}, 
\begin{align}
\|\overline{w}_{s}^{\tau E+t} -\overline{w}_s^{\tau E}\|^2
&=\| \sum_{t'=0}^{t-1}\eta_{\tau,s} \sum_{v=1}^V P_{v,s}^\tau g_{v,s,\tau}^{t'} \|^2 \\
&=\eta_{\tau,s}^2 \| \sum_{v=1}^V  P_{v,s}^\tau \sum_{t'=0}^{t-1} g_{v,s,\tau}^{t'}\|^2 \\
&=\eta_{\tau,s}^2 \| \sum_{v=1}^V \frac{d_{i,s}}{B_i} \left(\frac{B_i P_{v,s}^\tau}{d_{i,s}} \sum_{t'=0}^{t-1} g_{v,s,\tau}^{t'}\right)\|^2 \\
&\leq \eta_{\tau,s}^2  \sum_{v=1}^V \frac{d_{i,s}}{B_i} \|\frac{B_i P_{v,s}^\tau}{d_{i,s}} \sum_{t'=0}^{t-1} g_{v,s,\tau}^{t'}\|^2 && \text{Jensen's inequality} \\
&=\eta_{\tau,s}^2  \sum_{v=1}^V \frac{B_i(P_{v,s}^\tau)^2}{d_{i,s}} \| \sum_{t'=0}^{t-1} g_{v,s,\tau}^{t'}\|^2 \label{eq: jensen}\\
&\leq \eta_{\tau,s}^2  \sum_{v=1}^V \frac{B_i(P_{v,s}^\tau)^2}{d_{i,s}} t \sum_{t'=0}^{t-1} \|  g_{v,s,\tau}^{t'}\|^2
&& \text{AM-QM inequality} \\
&\leq \eta_{\tau,s}^2 E^2 G^2 \sum_{v=1}^V \frac{B_i(P_{v,s}^\tau)^2}{d_{i,s}}\label{eq:l2first}
\end{align}
Similarly, bound the second term of Eq. \ref{eq:l3}, 
\begin{align}
\sum_{v=1}^V P_{v,s}^\tau \| w_{v,s,\tau}^t- \overline{w}_s^{\tau E} \|^2
&=\sum_{v=1}^V P_{v,s}^\tau \| \eta_{\tau,s} \sum_{t'=0}^{t-1} g_{v,s,\tau}^{t'} \|^2 \\
&\leq \eta_{\tau,s}^2 \sum_{v=1}^V P_{v,s}^\tau  t \sum_{t'=0}^{t-1} \|  g_{v,s,\tau}^{t'}\|^2
&& \text{AM-QM inequality} \\
&\leq \eta_{\tau,s}^2 E^2 G^2 \sum_{v=1}^V P_{v,s}^\tau \label{eq:l2seccond}
\end{align}
Plug Eq. \ref{eq:l2first} and Eq. \ref{eq:l2seccond} to Eq. \ref{eq:l3}. And based on our assumption about the bound $\theta$ for the aggregation coefficient: $\textcolor{red}{B_i} P_{v,s}^\tau / d_{i,s} \leq \theta$, we have
\begin{align}
\sum_{v=1}^V P_{v,s}^\tau \|\overline{w}_{s}^{\tau E+t} -w_{v,s,\tau}^t\|^2
&=(\sum_{v=1}^V P_{v,s}^\tau -2)\|\overline{w}_{s}^{\tau E+t} -\overline{w}_s^{\tau E}\|^2+\sum_{v=1}^V P_{v,s}^\tau \| w_{v,s,\tau}^t- \overline{w}_s^{\tau E} \|^2 \\
&\leq  (\sum_{v=1}^V P_{v,s}^\tau -2)_+ 
\left(\eta_{\tau,s}^2 E^2 G^2 \sum_{v=1}^V \frac{B_i(P_{v,s}^\tau)^2}{d_{i,s}}\right)
+ 
\eta_{\tau,s}^2 E^2 G^2 \sum_{v=1}^V P_{v,s}^\tau\\
&=\eta_{\tau,s}^2 E^2 G^2 \left(
(\sum_{v=1}^V P_{v,s}^\tau -2)_+ (\sum_{v=1}^V \frac{\textcolor{red}{B_i}(P_{v,s}^\tau)^2}{d_{i,s}})
+
\sum_{v=1}^V P_{v,s}^\tau
\right)\\
&\leq \eta_{\tau,s}^2 E^2 G^2 \left(
\theta (\sum_{v=1}^V P_{v,s}^\tau -2)_+ (\sum_{v=1}^V P_{v,s}^\tau)
+
\sum_{v=1}^V P_{v,s}^\tau
\right)\\
&\leq \eta_{\tau,s}^2 E^2 G^2 \left(
\theta \left((\sum_{v=1}^V P_{v,s}^\tau)^2 -2(\sum_{v=1}^V P_{v,s}^\tau)+1\right)_+ 
+
\sum_{v=1}^V P_{v,s}^\tau
\right)\\
&\leq \eta_{\tau,s}^2 E^2 G^2 \left(
\theta \left(\sum_{v=1}^V P_{v,s}^\tau -1\right)^2 
+
\sum_{v=1}^V P_{v,s}^\tau
\right)
\end{align}
\end{proof}

\textbf{Bounding} $\|\overline{w}_s^{\tau E +t+1}-w_s^*\|^2$
\begin{align}
\|
\overline{w}_s^{\tau E +t+1}-w_s^*\|^2
&=\|\overline{w}_s^{\tau E+t}-\eta_{\tau,s} \sum_{v=1}^V P_{v,s}^\tau g_{v,s,\tau}^t - w_s^*\|^2\\
&=\|\overline{w}_s^{\tau E+t}-\eta_{\tau,s} \sum_{v=1}^V P_{v,s}^\tau g_{v,s,\tau}^t - w_s^*+\eta_{\tau,s} \sum_{v=1}^V P_{v,s}^\tau \overline{g}_{v,s,\tau}^t-\eta_{\tau,s} \sum_{v=1}^V P_{v,s}^\tau \overline{g}_{v,s,\tau}^t\|^2\\
&=
\underbrace{\|
\overline{w}_s^{\tau E+t}-w_s^*-\eta_{\tau,s} \sum_{v=1}^V P_{v,s}^\tau \overline{g}_{v,s,\tau}^t
\|^2}_{A_1}
+\underbrace{\eta_{\tau,s}^2\|
\sum_{v=1}^V P_{v,s}^\tau (\overline{g}_{v,s,\tau}^t-g_{v,s,\tau}^t)
\|^2}_{A_2}\\
&+\underbrace{2\eta_{\tau,s} \langle
\overline{w}_s^{\tau E+t}-w_s^*-\eta_{\tau,s} \sum_{v=1}^V P_{v,s}^\tau \overline{g}_{v,s,\tau}^t, 
\sum_{v=1}^V P_{v,s}^\tau (\overline{g}_{v,s,\tau}^t-g_{v,s,\tau}^t)
\rangle}_{A_3} \label{eq:goal}
\end{align}
Since $\mathbb{E}[g_{v,s,\tau}^t]=\overline{g}_{v,s,\tau}^t$, we know $\mathbb{E}_\xi [A_3]=0$. Now we first bound $A_1$. 
\begin{align}
A_1&=\|
\overline{w}_s^{\tau E+t}-w_s^*-\eta_{\tau,s} \sum_{v=1}^V P_{v,s}^\tau \overline{g}_{v,s,\tau}^t
\|^2\\
&=\|
\overline{w}_s^{\tau E +t}-w_s^*
\|^2
\underbrace{
-2\eta_{\tau,s} \langle \overline{w}_s^{\tau E +t}-w_s^*, 
\sum_{v=1}^V P_{v,s}^\tau \overline{g}_{v,s,\tau}^t
\rangle}_{B_1}
+
\underbrace{\eta_{\tau,s}^2 \|\sum_{v=1}^V P_{v,s}^\tau \overline{g}_{v,s,\tau}^t\|^2}_{B_2} 
\end{align}
And we know 
\begin{align}
B_2&=\eta_{\tau,s}^2 \|\sum_{v=1}^V P_{v,s}^\tau \overline{g}_{v,s,\tau}^t\|^2\\ 
&=\eta_{\tau,s}^2 \|\sum_{v=1}^V \frac{d_{i,s}}{B_i} \frac{B_i P_{v,s}^\tau}{d_{i,s}} \overline{g}_{v,s,\tau}^t\|^2\\
&\leq \eta_{\tau,s}^2 \sum_{v=1}^V \frac{d_{i,s}}{B_i} \|\frac{B_i P_{v,s}^\tau}{d_{i,s}} \overline{g}_{v,s,\tau}^t\|^2 && \text{Jensen's inequality} \\
&= \eta_{\tau,s}^2 \sum_{v=1}^V \frac{B_i(P_{v,s}^\tau)^2}{d_{i,s}} \|\overline{g}_{v,s,\tau}^t\|^2
\end{align}

Here we define the minimum of $\frac{d_{i,s}}{B_i}$ for each task to be $d_s$. 
\begin{align}
B_2\leq 
\eta_{\tau,s}^2 \sum_{v=1}^V \frac{B_i (P_{v,s}^\tau)^2}{d_{i,s}} \| \overline{g}_{v,s,\tau}^t \|^2 
\leq 
\eta_{\tau,s}^2 \frac{1}{d_s} \sum_{v=1}^V (P_{v,s}^\tau)^2 \| \overline{g}_{v,s,\tau}^t \|^2
\label{eq:b2}
\end{align}
\begin{align}
B_1&= -2\eta_{\tau,s} \langle \overline{w}_s^{\tau E +t}-w_s^*, 
\sum_{v=1}^V P_{v,s}^\tau \overline{g}_{v,s,\tau}^t
\rangle \\
&=-2\eta_{\tau,s}\sum_{v=1}^V  \langle \overline{w}_s^{\tau E +t}-w_s^*, 
P_{v,s}^\tau \overline{g}_{v,s,\tau}^t
\rangle \\
&= -2\eta_{\tau,s}\sum_{v=1}^V  \langle \overline{w}_s^{\tau E +t}-w_{v,s,\tau}^t+w_{v,s,\tau}^t-w_s^*, 
P_{v,s}^\tau \overline{g}_{v,s,\tau}^t
\rangle \\
&=\underbrace{
-2\eta_{\tau,s}\sum_{v=1}^V P_{v,s}^\tau \langle \overline{w}_s^{\tau E +t}-w_{v,s,\tau}^t, 
\overline{g}_{v,s,\tau}^t
\rangle}_{C_1} \underbrace{
-2\eta_{\tau,s} \sum_{v=1}^V P_{v,s}^\tau \langle w_{v,s,\tau}^t-w_s^*, \overline{g}_{v,s,\tau}^t
\rangle}_{C_2}
\end{align}
We bound $C_1$, 
\begin{align}
C_1 &= -2\eta_{\tau,s}\sum_{v=1}^V P_{v,s}^\tau \langle \overline{w}_s^{\tau E +t}-w_{v,s,\tau}^t, 
 \overline{g}_{v,s,\tau}^t
\rangle \\
& \leq | -2\eta_{\tau,s}\sum_{v=1}^V P_{v,s}^\tau \langle \overline{w}_s^{\tau E +t}-w_{v,s,\tau}^t, 
 \overline{g}_{v,s,\tau}^t
\rangle| \\
&= 2\eta_{\tau,s}\sum_{v=1}^V P_{v,s}^\tau |\langle \overline{w}_s^{\tau E +t}-w_{v,s,\tau}^t, 
 \overline{g}_{v,s,\tau}^t
\rangle|\\
&\leq 2\eta_{\tau,s}\sum_{v=1}^V P_{v,s}^\tau \| \overline{w}_s^{\tau E +t}-w_{v,s,\tau}^t\|
\| \overline{g}_{v,s,\tau}^t
\| \\
&\leq \eta_{\tau,s}\sum_{v=1}^V P_{v,s}^\tau (\frac{1}{\eta_{\tau,s}} \| \overline{w}_s^{\tau E +t}-w_{v,s,\tau}^t\|^2
+
\eta_{\tau,s} \| \overline{g}_{v,s,\tau}^t\|^2 ) \label{eq:c1}
\end{align}
Then we bound $C_2$. Since we assume $f_{i,s}$ is $\mu$-strongly convex, 
\begin{align}
\langle w_{v,s,\tau}^t -w_s^*, \overline{g}_{v,s,\tau}^t\rangle \geq (f_{i,s}(w_{v,s,\tau}^t)-f_{i,s}(w_s^*))+\frac{\mu}{2}\|w_{v,s,\tau}^t-w_s^*\|^2
\end{align}
we know
\begin{align}
C_2&=-2\eta_{\tau,s} \sum_{v=1}^V P_{v,s}^\tau \langle w_{v,s,\tau}^t-w_s^*, \overline{g}_{v,s,\tau}^t
\rangle \\
&\leq -2\eta_{\tau,s} \sum_{v=1}^V P_{v,s}^\tau (  (f_{i,s}(w_{v,s,\tau}^t)-f_{i,s}(w_s^*))+\frac{\mu}{2}\|w_{v,s,\tau}^t-w_s^*\|^2 ) \label{eq:c2}
\end{align}
Combine Eq. \ref{eq:c1} and \ref{eq:c2}, we can bound $B_1$
\begin{align}
B_1\leq \sum_{v=1}^V P_{v,s}^\tau \left(
\|\overline{w}_s^{\tau E +t}-w_{v,s,\tau}^t\|^2
+\eta_{\tau,s}^2\|\overline{g}_{v,s,\tau}^t\|^2
-2\eta_{\tau,s} \left(
(f_{i,s}(w_{v,s,\tau}^t)-f_{i,s}(w_s^*))+\frac{\mu}{2}\|w_{v,s,\tau}^t-w_s^*\|^2
\right)
\right)\label{eq:b1}
\end{align}
Combine Eq. \ref{eq:b1} and Eq. \ref{eq:b2}, we can bound $A_1$, 
\begin{align}
A_1&=\|\overline{w}_s^{\tau E +t}-w_s^*\|^2+B_1+B_2\\
&\leq \|\overline{w}_s^{\tau E +t}-w_s^*\|^2+\eta_{\tau,s}^2 \frac{1}{d_s} \sum_{v=1}^V (P_{v,s}^\tau)^2 \| \overline{g}_{v,s,\tau}^t \|^2 \\ 
&+\sum_{v=1}^V P_{v,s}^\tau \left(
\|\overline{w}_s^{\tau E +t}-w_{v,s,\tau}^t\|^2
+\eta_{\tau,s}^2\|\overline{g}_{v,s,\tau}^t\|^2
-2\eta_{\tau,s} \left(
(f_{i,s}(w_{v,s,\tau}^t)-f_{i,s}(w_s^*))+\frac{\mu}{2}\|w_{v,s,\tau}^t-w_s^*\|^2
\right)
\right)
\end{align}
Since $f_{i,s}$ is $L$-smooth, 
\begin{align}
\|\overline{g}_{v,s,\tau}^t\|^2 \leq 2L(f_{i,s}(w_{v,s,\tau}^t)-f_{i,s}(w_{i,s}^*))
\end{align}
Define $f_{v,s,\tau}^t=f_{i,s}(w_{v,s,\tau}^t)$, $f_{i,s}^*=f_{i,s}(w_{i,s}^*)$. We know
\begin{align}
A_1&\leq \|\overline{w}_s^{\tau E +t}-w_s^*\|^2+\eta_{\tau,s}^2 \frac{1}{d_s} \sum_{v=1}^V (P_{v,s}^\tau)^2 \| \overline{g}_{v,s,\tau}^t \|^2 \\ 
&+\sum_{v=1}^V P_{v,s}^\tau \left(
\|\overline{w}_s^{\tau E +t}-w_{v,s,\tau}^t\|^2
+2L\eta_{\tau,s}^2(f_{v,s,\tau}^t-f_{i,s}^*)
-2\eta_{\tau,s} \left(
(f_{v,s,\tau}^t-f_{i}(w_s^*))+\frac{\mu}{2}\|w_{v,s,\tau}^t-w_s^*\|^2
\right)
\right)\\
&= \|\overline{w}_s^{\tau E +t}-w_s^*\|^2
+\eta_{\tau,s}^2 \frac{1}{d_s} \sum_{v=1}^V (P_{v,s}^\tau)^2 \| \overline{g}_{v,s,\tau}^t \|^2
+\sum_{v=1}^V P_{v,s}^\tau \|\overline{w}_s^{\tau E +t}-w_{v,s,\tau}^t\|^2
-\eta_{\tau,s} \mu \sum_{v=1}^V P_{v,s}^\tau \|w_{v,s,\tau}^t-w_s^*\|^2
\\
&+ \underbrace{2L\eta_{\tau,s}^2\sum_{v=1}^V P_{v,s}^\tau(f_{v,s,\tau}^t-f_{i,s}^*)-2\eta_{\tau,s} \sum_{v=1}^V P_{v,s}^\tau (f_{v,s,\tau}^t-f_{i}(w_s^*))}_{C} \label{eq:a1}
\end{align}
Notice
\begin{align}
\|w_{v,s,\tau}^t-w_s^*\|^2
&=\|w_{v,s,\tau}^t-\overline{w}_s^{\tau E+t}+\overline{w}_s^{\tau E+t}-w_s^*\|^2\\
&=\|w_{v,s,\tau}^t-\overline{w}_s^{\tau E+t}\|^2+\|\overline{w}_s^{\tau E+t}-w_s^*\|^2+2\langle w_{v,s,\tau}^t-\overline{w}_s^{\tau E+t}, \overline{w}_s^{\tau E+t}-w_s^*\rangle \\
&\geq \|w_{v,s,\tau}^t-\overline{w}_s^{\tau E+t}\|^2+\|\overline{w}_s^{\tau E+t}-w_s^*\|^2
-2\| w_{v,s,\tau}^t-\overline{w}_s^{\tau E+t}\| \| \overline{w}_s^{\tau E+t}-w_s^*\|\\
&\geq \|w_{v,s,\tau}^t-\overline{w}_s^{\tau E+t}\|^2+\|\overline{w}_s^{\tau E+t}-w_s^*\|^2
-(2\| w_{v,s,\tau}^t-\overline{w}_s^{\tau E+t}\|^2+ \frac{1}{2}\| \overline{w}_s^{\tau E+t}-w_s^*\|^2)\\
&=\frac{1}{2}\| \overline{w}_s^{\tau E+t}-w_s^*\|^2 - \| w_{v,s,\tau}^t-\overline{w}_s^{\tau E+t}\|^2
\end{align}
Thus, 
\begin{align}
A_1&\leq \|\overline{w}_s^{\tau E +t}-w_s^*\|^2
+\eta_{\tau,s}^2 \frac{1}{d_s} \sum_{v=1}^V (P_{v,s}^\tau)^2 \| \overline{g}_{v,s,\tau}^t \|^2
+\sum_{v=1}^V P_{v,s}^\tau \|\overline{w}_s^{\tau E +t}-w_{v,s,\tau}^t\|^2\\
&-\eta_{\tau,s} \mu \sum_{v=1}^V P_{v,s}^\tau \|w_{v,s,\tau}^t-w_s^*\|^2+C \\
&= \|\overline{w}_s^{\tau E +t}-w_s^*\|^2
+\eta_{\tau,s}^2 \frac{1}{d_s} \sum_{v=1}^V (P_{v,s}^\tau)^2 \| \overline{g}_{v,s,\tau}^t \|^2
+\sum_{v=1}^V P_{v,s}^\tau \|\overline{w}_s^{\tau E +t}-w_{v,s,\tau}^t\|^2
+C \\
&-\eta_{\tau,s} \mu \sum_{v=1}^V P_{v,s}^\tau (\frac{1}{2}\| \overline{w}_s^{\tau E+t}-w_s^*\|^2 - \| w_{v,s,\tau}^t-\overline{w}_s^{\tau E+t}\|^2)\\
&= (1-\frac{1}{2}\eta_{\tau,s} \mu \sum_{v=1}^V P_{v,s}^\tau)\| \overline{w}_s^{\tau E+t}-w_s^*\|^2
+ (1+\eta_{\tau,s} \mu) \sum_{v=1}^V P_{v,s}^\tau \|\overline{w}_s^{\tau E +t}-w_{v,s,\tau}^t\|^2\\
&+ \eta_{\tau,s}^2 \frac{1}{d_s} \sum_{v=1}^V (P_{v,s}^\tau)^2 \| \overline{g}_{v,s,\tau}^t \|^2+C
\end{align}
\begin{align}
C&=2L\eta_{\tau,s}^2\sum_{v=1}^V P_{v,s}^\tau(f_{v,s,\tau}^t-f_{i,s}^*)-2\eta_{\tau,s} \sum_{v=1}^V P_{v,s}^\tau (f_{v,s,\tau}^t-f_{i}(w_s^*))\\
&=-2\eta_{\tau,s} (1-L\eta_{\tau,s}) \sum_{v=1}^V P_{v,s}^\tau(f_{v,s,\tau}^t-f_{i,s}^*) + 2\eta_{\tau,s} \sum_{v=1}^V P_{v,s}^\tau (f_{i,s}(w_s^*)-f_{i,s}^*)
\end{align}
Let $\gamma_\tau = 2\eta_{\tau,s} (1-L\eta_{\tau,s})$. Assume $\eta_{\tau,s} \leq \frac{1}{2L}$, hence $\eta_{\tau,s}\leq \gamma_\tau \leq 2\eta_{\tau,s}$. 
\begin{align}
C&=-2\eta_{\tau,s} (1-L\eta_{\tau,s}) \sum_{v=1}^V P_{v,s}^\tau(f_{v,s,\tau}^t-f_{i,s}^*) + 2\eta_{\tau,s} \sum_{v=1}^V P_{v,s}^\tau (f_{i,s}(w_s^*)-f_{i,s}^*) \\
&=-\gamma_\tau \sum_{v=1}^V P_{v,s}^\tau(f_{v,s,\tau}^t-f_{i,s}^*) + 2\eta_{\tau,s} \sum_{v=1}^V P_{v,s}^\tau (f_{i,s}(w_s^*)-f_{i,s}^*) \\
&=-\gamma_\tau \sum_{v=1}^V P_{v,s}^\tau(f_{v,s,\tau}^t-f_{i,s}^*+f_{i,s}(w_s^*)-f_{i,s}(w_s^*)) + 2\eta_{\tau,s} \sum_{v=1}^V P_{v,s}^\tau (f_{i,s}(w_s^*)-f_{i,s}^*)\\
&=-\gamma_\tau \sum_{v=1}^V P_{v,s}^\tau(f_{v,s,\tau}^t-f_{i,s}(w_s^*)) + (2\eta_{\tau,s}-\gamma_\tau) \sum_{v=1}^V P_{v,s}^\tau (f_{i,s}(w_s^*)-f_{i,s}^*) \\
&= \underbrace{-\gamma_\tau \sum_{v=1}^V P_{v,s}^\tau(f_{v,s,\tau}^t-f_{i,s}(w_s^*))}_{D} + 2L\eta_{\tau,s}^2 \sum_{v=1}^V P_{v,s}^\tau \Gamma_{i,s}
\end{align}
Then we bound $D=-\gamma_\tau \sum_{v=1}^V P_{v,s}^\tau(f_{v,s,\tau}^t-f_{i,s}(w_s^*))$
\begin{align}
& \sum_{v=1}^V P_{v,s}^\tau(f_{v,s,\tau}^t-f_{i,s}(w_s^*))
=\sum_{v=1}^V P_{v,s}^\tau(f_{v,s,\tau}^t-f_{i,s}(w_s^*)+f_{i,s}(\overline{w}_s^{\tau E+t})-f_{i,s}(\overline{w}_s^{\tau E+t})) \\
&=\sum_{v=1}^V P_{v,s}^\tau (f_{v,s,\tau}^t-f_{i,s}(\overline{w}_s^{\tau E+t})) + \sum_{v=1}^V P_{v,s}^\tau (f_{i,s}(\overline{w}_s^{\tau E+t})-f_{i,s}(w_s^*)) \\
&\geq \sum_{v=1}^V P_{v,s}^\tau 
\langle 
\nabla f_{i,s}(\overline{w}_s^{\tau E+t}), 
w_{v,s,\tau}^t - \overline{w}_s^{\tau E+t}
\rangle
+ \sum_{v=1}^V P_{v,s}^\tau (f_{i,s}(\overline{w}_s^{\tau E+t})-f_{i,s}(w_s^*))\\
&\geq -\sum_{v=1}^V P_{v,s}^\tau 
\|
\nabla f_{i,s}(\overline{w}_s^{\tau E+t})\| \| 
w_{v,s,\tau}^t - \overline{w}_s^{\tau E+t}
\|
+ \sum_{v=1}^V P_{v,s}^\tau (f_{i,s}(\overline{w}_s^{\tau E+t})-f_{i,s}(w_s^*))\\
&\geq -\frac{1}{2}\sum_{v=1}^V P_{v,s}^\tau 
(\eta_{\tau,s} \underbrace{\|\nabla f_{i,s}(\overline{w}_s^{\tau E+t})\|^2}_{\leq 2L(f_{i,s}(\overline{w}_{s}^{\tau E+t})-f_{i,s}^*)} 
+\frac{1}{\eta_{\tau,s}}\| w_{v,s,\tau}^t - \overline{w}_s^{\tau E+t}\|^2)
+ \sum_{v=1}^V P_{v,s}^\tau (f_{i,s}(\overline{w}_s^{\tau E+t})-f_{i,s}(w_s^*)) \\ 
&\geq - \sum_{v=1}^V P_{v,s}^\tau 
(\eta_{\tau,s} L (f_{i,s}(\overline{w}_{s}^{\tau E+t})-f_{i,s}^*) 
+\frac{1}{2\eta_{\tau,s}}\| w_{v,s,\tau}^t - \overline{w}_s^{\tau E+t}\|^2)
+ \sum_{v=1}^V P_{v,s}^\tau (f_{i,s}(\overline{w}_s^{\tau E+t})-f_{i,s}(w_s^*))
\end{align}
Thus, 
\begin{align}
D&\leq \gamma_\tau \sum_{v=1}^V P_{v,s}^\tau 
(\eta_{\tau,s} L \underbrace{(f_{i,s}(\overline{w}_{s}^{\tau E+t})-f_{i,s}^*)}_{f_{i,s}(\overline{w}_{s}^{\tau E+t})-f_{i,s}^*-f_{i,s}(w_s^*)+f_{i,s}(w_s^*)} 
+\frac{1}{2\eta_{\tau,s}}\| w_{v,s,\tau}^t - \overline{w}_s^{\tau E+t}\|^2)\\
&-\gamma_\tau \sum_{v=1}^V P_{v,s}^\tau (f_{i,s}(\overline{w}_s^{\tau E+t})-f_{i,s}(w_s^*))\\
&= \gamma_\tau (\eta_{\tau,s} L-1)\sum_{v=1}^V P_{v,s}^\tau 
(f_{i,s}(\overline{w}_s^{\tau E+t})-f_{i,s}(w_s^*)) 
+\sum_{v=1}^V P_{v,s}^\tau \underbrace{\frac{\gamma_\tau}{2\eta_{\tau,s}}}_{\leq 1}\| w_{v,s,\tau}^t - \overline{w}_s^{\tau E+t}\|^2
+\underbrace{\gamma_\tau}_{\leq 2\eta_{\tau,s}} \eta_{\tau,s} L \sum_{v=1}^V P_{v,s}^\tau \Gamma_{i,s} \\
&\leq \gamma_\tau (\eta_{\tau,s} L-1)\sum_{v=1}^V P_{v,s}^\tau 
(f_{i,s}(\overline{w}_s^{\tau E+t})-f_{i,s}(w_s^*)) 
+\sum_{v=1}^V P_{v,s}^\tau \| w_{v,s,\tau}^t - \overline{w}_s^{\tau E+t}\|^2
+2\eta_{\tau,s}^2 L \sum_{v=1}^V P_{v,s}^\tau \Gamma_{i,s}
\end{align}
Therefore, 
\begin{align}
C&=D + 2L\eta_{\tau,s}^2 \sum_{v=1}^V P_{v,s}^\tau \Gamma_{i,s} \\
&\leq \gamma_\tau (\eta_{\tau,s} L-1)\sum_{v=1}^V P_{v,s}^\tau 
(f_{i,s}(\overline{w}_s^{\tau E+t})-f_{i,s}(w_s^*)) 
+\sum_{v=1}^V P_{v,s}^\tau \| w_{v,s,\tau}^t - \overline{w}_s^{\tau E+t}\|^2
+4\eta_{\tau,s}^2 L \sum_{v=1}^V P_{v,s}^\tau \Gamma_{i,s}\label{eq:c}
\end{align}
Thus, we can use the upper bound of $C$ to bound $A_1$
\begin{align}
A_1&\leq 
(1-\frac{1}{2}\eta_{\tau,s} \mu \sum_{v=1}^V P_{v,s}^\tau)\| \overline{w}_s^{\tau E+t}-w_s^*\|^2
+ (1+\eta_{\tau,s} \mu) \sum_{v=1}^V P_{v,s}^\tau \|\overline{w}_s^{\tau E +t}-w_{v,s,\tau}^t\|^2
+ \eta_{\tau,s}^2 \frac{1}{d_s} \sum_{v=1}^V (P_{v,s}^\tau)^2 \| \overline{g}_{v,s,\tau}^t \|^2+C \\
&\leq 
(1-\frac{1}{2}\eta_{\tau,s} \mu \sum_{v=1}^V P_{v,s}^\tau)\| \overline{w}_s^{\tau E+t}-w_s^*\|^2
+ (2+\eta_{\tau,s} \mu) \sum_{v=1}^V P_{v,s}^\tau \|\overline{w}_s^{\tau E +t}-w_{v,s,\tau}^t\|^2
+ \eta_{\tau,s}^2 \frac{1}{d_s} \sum_{v=1}^V (P_{v,s}^\tau)^2 \| \overline{g}_{v,s,\tau}^t \|^2 \\
&+\gamma_\tau (\eta_{\tau,s} L-1)\sum_{v=1}^V P_{v,s}^\tau 
(f_{i,s}(\overline{w}_s^{\tau E+t})-f_{i,s}(w_s^*)) 
+4\eta_{\tau,s}^2 L \sum_{v=1}^V P_{v,s}^\tau \Gamma_{i,s} \label{eq:a1}
\end{align}
Plug Eq. \ref{eq:a1} to Eq. \ref{eq:goal}, 
\begin{align}
\|
& \overline{w}_s^{\tau E +t+1}-w_{s}^*
\|^2
=A_1+A_2
=A_1+\eta_{\tau,s}^2\|
\sum_{v=1}^V P_{v,s}^\tau (\overline{g}_{v,s,\tau}^t-g_{v,s,\tau}^t)
\|^2\\
&\leq 
(1-\frac{1}{2}\eta_{\tau,s} \mu \sum_{v=1}^V P_{v,s}^\tau)\| \overline{w}_s^{\tau E+t}-w_s^*\|^2
+ \underbrace{(2+\eta_{\tau,s} \mu)}_{\leq 2+\frac{\mu}{2L}} \sum_{v=1}^V P_{v,s}^\tau \|\overline{w}_s^{\tau E +t}-w_{v,s,\tau}^t\|^2
+ \eta_{\tau,s}^2 \frac{1}{d_s} \sum_{v=1}^V (P_{v,s}^\tau)^2 \| \overline{g}_{v,s,\tau}^t \|^2 \\
&+\underbrace{\gamma_\tau (\eta_{\tau,s} L-1)}_{\leq 2\eta_{\tau,s}}\sum_{v=1}^V P_{v,s}^\tau 
(f_{i,s}(w_s^*)-f_{i,s}(\overline{w}_s^{\tau E+t})) 
+4\eta_{\tau,s}^2 L \sum_{v=1}^V P_{v,s}^\tau \Gamma_{i,s} 
+\eta_{\tau,s}^2\|
\sum_{v=1}^V P_{v,s}^\tau (\overline{g}_{v,s,\tau}^t-g_{v,s,\tau}^t)\|^2 \label{eq:originalbound}
\end{align}
Define
\begin{align}
Q_{\tau E+t}&=(2+\frac{\mu}{2L}) \underbrace{\sum_{v=1}^V P_{v,s}^\tau \|\overline{w}_s^{\tau E +t}-w_{v,s,\tau}^t\|^2}_{\text{apply lemma \ref{lemma3}}}
+\eta_{\tau,s}^2\underbrace{\|
\sum_{v=1}^V P_{v,s}^\tau (\overline{g}_{v,s,\tau}^t-g_{v,s,\tau}^t)\|^2}_{\text{apply lemma \ref{lemma2}}}\\
&+4\eta_{\tau,s}^2 L \sum_{v=1}^V P_{v,s}^\tau \Gamma_{i,s}+\eta_{\tau,s}^2 \frac{1}{d_s} \sum_{v=1}^V (P_{v,s}^\tau)^2 \| \overline{g}_{v,s,\tau}^t \|^2
\end{align}
Apply the lemma \ref{lemma2} and \ref{lemma3} we have
\begin{align}
\mathbb{E}_\xi [Q_{\tau E+t}]
&\leq (2+\frac{\mu}{2L})\eta_{\tau,s}^2 E^2 G^2 \left(
\theta \left(\sum_{v=1}^V P_{v,s}^\tau -1\right)^2 
+\sum_{v=1}^V P_{v,s}^\tau
\right)\\
&+\eta_{\tau,s}^2 \sum_{v=1}^V (P_{v,s}^\tau)^2 \sigma_{i,s}^2+4\eta_{\tau,s}^2 L \sum_{v=1}^V P_{v,s}^\tau \Gamma_{i,s}+\eta_{\tau,s}^2 \frac{1}{d_s} \sum_{v=1}^V (P_{v,s}^\tau)^2 \| \overline{g}_{v,s,\tau}^t \|^2
\end{align}
\textcolor{red}{Notice $\overline{g}_{v,s,\tau}^t=\overline{g}_{j,s,\tau}^t$ for processor $v$ and $j$ within the same client $i$. 
}
And define
\begin{align}
B_{\tau E+t}&=\frac{1}{\eta_{\tau,s}^2}\mathbb{E}_\xi [Q_{\tau E+t}]\leq (2+\frac{\mu}{2L}) E^2 G^2 \left(
\theta \left(\sum_{v=1}^V P_{v,s}^\tau -1\right)^2 
+\sum_{v=1}^V P_{v,s}^\tau
\right)\\
&+ \sum_{v=1}^V (P_{v,s}^\tau)^2 \sigma_{i,s}^2+4 L \sum_{v=1}^V P_{v,s}^\tau \Gamma_{i,s}+ \frac{1}{d_s} \sum_{v=1}^V (P_{v,s}^\tau)^2 \| \overline{g}_{v,s,\tau}^t \|^2
\end{align}
\textcolor{blue}{Now we find an upper bound for $\frac{\mu}{2L}$ to make our loss-based sampling optimization objective appear in the whole upper bound of the convergence. }

With $L$ smooth assumption, 
\begin{align}
\|\overline{g}_{v,s,\tau}^0\|^2 \leq 2L(f_{v,s,\tau}^0-f_{i,s}^* )
\end{align}
Thus, for any $v,s,\tau$, 
\begin{align}
\frac{1}{2L}\leq \frac{f_{v,s,\tau}^0-f_{i,s}^*}{\|\overline{g}_{v,s,\tau}^0\|^2}
\leq \frac{f_{v,s,\tau}^0}{\|\overline{g}_{v,s,\tau}^0\|^2}
\end{align}
With $\mu$ convexity, 
\begin{align}
f_{i,s}(w_{v,s,\tau}^0) &\geq f_{i,s}(w_{i,s}^*) + \nabla f_{i,s}(w_{i,s}^*)^\top (w_{v,s,\tau}^0-w_{i,s}^*)+\frac{\mu}{2} \| w_{v,s,\tau}^0-w_{i,s}^*\|^2\\
&=f_{i,s}(w_{i,s}^*) +\frac{\mu}{2} \| w_{v,s,\tau}^0-w_{i,s}^*\|^2
\end{align}
Thus, for any $v,s,\tau$, 
\begin{align}
\mu \leq \frac{2(f_{v,s,\tau}^0-f_{i,s}^*)}{\| w_{v,s,\tau}^0-w_{i,s}^*\|^2}
\leq \frac{2f_{v,s,\tau}^0}{\| w_{v,s,\tau}^0-w_{i,s}^*\|^2}
\end{align}
Therefore, 
\begin{align}
\frac{\mu}{2L}&\leq
\frac{2(f_{v,s,\tau}^0)^2}{\| w_{v,s,\tau}^0-w_{i,s}^*\|^2\|\overline{g}_{v,s,\tau}^0\|^2}
\end{align}
\textcolor{blue}{$g_{v,s,\tau}^0$ is the gradient at the first local epoch within a global round, ie when the aggregated global model becomes the client's local model, and the client trains with it.}
In non-iid settings, \textcolor{blue}{it is likely that the returned global model at the first local epoch round  cannot reach the client's own local optimal because of noniid-ness.} 
This means we can find \textcolor{blue}{$e_f$ such that } $\|\overline{g}_{v,s,\tau}^0\|>e_f>0$, where  $e_f$ reflects the non-iid level. Similarly, we can find $\|w_{v,s,\tau}^0-w_{i,s}^*\|>e_w>0$. 
Thus, for any $i,s,\tau$
\begin{align}
\frac{\mu}{2L}&\leq
\frac{2(f_{v,s,\tau}^0)^2}{e_w^2e_f^2} \end{align}
Since this inequality holds for any $i,s,\tau$, we can choose the minimum of $f_{v,s,\tau}^0$
\begin{align}
\frac{\mu}{2L}&\leq
\frac{2}{e_w^2e_f^2}\min[(f_{v,s,\tau}^0)^2]\\
&= \frac{2}{e_w^2e_f^2}\frac{V}{ (\frac{V}{\min[(f_{v,s,\tau}^0)^2]})}\\
&= \frac{2}{e_w^2e_f^2}\frac{V}{\sum_{j=1}^V \frac{1}{\min[(f_{v,s,\tau}^0)^2]}}\\
&\leq \frac{2}{e_w^2e_f^2}\frac{V}{\sum_{v=1}^V \frac{1}{(f_{v,s,\tau}^0)^2}}
\end{align}
Notice by Titu's lemma (Cauchy-Schwarz inequality), we have
\begin{align}
\frac{V}{\sum_{v=1}^V \frac{1}{(f_{v,s,\tau}^0)^2}}(\sum_{v=1}^V P_{v,s}^\tau -1)^2 &=
\frac{V}{\sum_{v=1}^V \frac{1}{(f_{v,s,\tau}^0)^2}}(\sum_{v=1}^V P_{v,s}^\tau -\sum_{v=1}^V \frac{d_{i,s}}{B_i})^2 \\
&= 
V\frac{(\sum_{v=1}^V P_{v,s}^\tau -\sum_{v=1}^V \frac{d_{i,s}}{B_i})^2 }{\sum_{v=1}^V \frac{1}{(f_{v,s,\tau}^0)^2}}\\
&\leq V \sum_{v=1}^V (P_{v,s}^\tau-\frac{d_{i,s}}{B_i})^2(f_{v,s,\tau}^0)^2 \\
&= V \sum_{v=1}^V (P_{v,s}^\tau f_{v,s,\tau}^0-\frac{d_{i,s}}{B_i}f_{v,s,\tau}^0)^2
\end{align}
Therefore, for any $i,s,\tau$
\begin{align}
B_{\tau E+t}&\leq (2+\frac{\mu}{2L}) E^2G^2 \sum_{v=1}^V P_{v,s}^\tau 
+E^2G^2(2+\frac{2\min[(f_{v,s,\tau}^0)^2]}{e_w^2e_f^2})\theta (\sum_{v=1}^V P_{v,s}^\tau -\sum_{v=1}^V \frac{d_{i,s}}{\textcolor{red}{B_i}})^2  \\
&+ \sum_{v=1}^V (P_{v,s}^\tau)^2 \sigma_{i,s}^2+4 L \sum_{v=1}^V P_{v,s}^\tau \Gamma_{i,s}+ \frac{1}{d_s} \sum_{v=1}^V (P_{v,s}^\tau)^2 \| \overline{g}_{v,s,\tau}^t \|^2 \\
&\leq (2+\frac{\mu}{2L}) E^2G^2 \sum_{v=1}^V P_{v,s}^\tau 
+2E^2G^2\theta (\sum_{v=1}^V P_{v,s}^\tau -\sum_{v=1}^V d_{i,s})^2
+\frac{2E^2G^2\theta}{e_w^2e_f^2} V  \sum_{v=1}^V (P_{v,s}^\tau f_{v,s,\tau}^0-\frac{d_{i,s}}{B_i}f_{v,s,\tau}^0)^2
\\
&+ \sum_{v=1}^V (P_{v,s}^\tau)^2 \sigma_{i,s}^2+4 L \sum_{v=1}^V P_{v,s}^\tau \Gamma_{i,s}+ \frac{1}{d_s} \sum_{v=1}^V (P_{v,s}^\tau)^2 \| \overline{g}_{v,s,\tau}^t \|^2 
\label{eq:b_bound}
\end{align}

By Eq. \ref{eq:originalbound}, 
\begin{align}
\|\overline{w}_s^{\tau E +t+1}-w_s^*\|^2 \leq 
(1-\frac{1}{2}\eta_{\tau,s} \mu \sum_{v=1}^V P_{v,s}^\tau)\| \overline{w}_s^{\tau E+t}-w_s^*\|^2+
2\eta_{\tau,s}\sum_{v=1}^V P_{v,s}^\tau (f_{i,s}(w_s^*)-f_{i,s}(\overline{w}_s^{\tau E+t}))
+\eta_{\tau,s}^2 B_{\tau E +t}
\end{align}
We write $\Delta_{\tau E+t}=\|\overline{w}_s^{\tau E +t+1}-w_s^*\|^2$ for convenience. 

\textbf{Bounding} $\|\overline{w}_s^{\tau E}-w_s^*\|$

Summing from $\tau E$ to $(\tau+1)E$ we have
\begin{align}
\sum_{t=1}^E \Delta_{\tau E+t} \leq \sum_{t=0}^{E-1} (1-\frac{1}{2}\eta_{\tau,s} \mu \sum_{v=1}^V P_{v,s}^\tau) \Delta_{\tau E+t}
+\eta_{\tau,s}^2 B_\tau +2\eta_{\tau,s}\sum_{v=1}^V P_{v,s}^\tau (f_{i,s}(w_s^*)-f_{i,s}(\overline{w}_s^{\tau E+l}))
\end{align}
where $B_\tau=\sum_{t=0}^{E-1} B_{\tau E+t}$, and $\overline{w}_s^{\tau E+l} = \mathrm{argmin}_{\overline{w}_s^{\tau E+t}} \sum_{v=1}^V P_{v,s}^\tau f_{i,s}(\overline{w}_s^{\tau E+t})$. 

Reorganize it we get
\begin{align}
\Delta_{(\tau+1)E}\leq \Delta_{\tau E} - \frac{1}{2}\eta_{\tau,s} \mu \sum_{t=0}^{E-1} \sum_{v=1}^V P_{v,s}^\tau \Delta_{\tau E +t}+\eta_{\tau,s}^2 B_\tau +2\eta_{\tau,s} E\sum_{v=1}^V P_{v,s}^\tau (f_{i,s}(w_s^*)-f_{i,s}(\overline{w}_s^{\tau E+l}))\label{eq:reorganize}
\end{align}
Now we seek to find a lower bound for $\Delta_{\tau E +t}$. 
\begin{align}
\sqrt{\Delta_{\tau E +t+1}} &= \|\overline{w}_{s}^{\tau E +t+1}-w_s^*\|\\
&=\|\overline{w}_{s}^{\tau E +t+1}-\overline{w}_s^{\tau E+t}+\overline{w}_s^{\tau E+t}-w_s^*\| \\
&\leq \|\overline{w}_{s}^{\tau E +t+1}-\overline{w}_s^{\tau E+t}\|+\sqrt{\Delta_{\tau E+t}} \\
&= \|\eta_{\tau,s} \sum_{v=1}^V P_{v,s}^\tau g_{v,s,\tau}^t\|+\sqrt{\Delta_{\tau E+t}}
\end{align}
Define $h_{\tau E+t}=\|\sum_{v=1}^V P_{v,s}^\tau g_{v,s,\tau}^t\|$. We know $\sqrt{\Delta_{\tau E +t+1}}\leq \eta_{\tau,s} h_{\tau E +t}+\sqrt{\Delta_{\tau E +t}}$. 
Thus, 
\begin{align}
\sqrt{\Delta_{(\tau+1) E}} &\leq  \sqrt{\Delta_{(\tau+1) E-1}}+\eta_{\tau,s} h_{(\tau+1) E-1} \\
&\leq \cdots \leq \sqrt{\Delta_{\tau E+t'}}+\sum_{t=t'}^{E-1} \eta_{\tau,s} h_{\tau E +t}
\end{align}
\begin{align}
\Delta_{(\tau +1)E}&\leq 
\Delta_{\tau E+t'} +2\sqrt{\Delta_{\tau E +t'}}(\sum_{t=t'}^{E-1} \eta_{\tau,s} h_{\tau E +t} )+(\sum_{t=t'}^{E-1} \eta_{\tau,s} h_{\tau E +t} )^2\\
&\leq 2\Delta_{\tau E+t'}+2(\sum_{t=t'}^{E-1} \eta_{\tau,s} h_{\tau E +t} )^2
\end{align}
Therefore, 
\begin{align}
\Delta_{\tau E+t'}&\geq \frac{1}{2} \Delta_{(\tau+1)E}-(\sum_{t=t'}^{E-1} \eta_{\tau,s} h_{\tau E +t} )^2\\
&\geq \frac{1}{2} \Delta_{(\tau+1)E}-(\sum_{t=0}^{E-1} \eta_{\tau,s} h_{\tau E +t} )^2 \label{eq:delta}
\end{align}
Plug Eq. \ref{eq:delta} to Eq. \ref{eq:reorganize}, we can get
\begin{align}
(1+\frac{1}{4}\mu\eta_{\tau,s} E\sum_{v=1}^V P_{v,s}^\tau )\Delta_{(\tau+1) E }&\leq \Delta_{\tau E}+\frac{1}{2}\mu \eta_{\tau,s}^3 E \sum_{v=1}^V P_{v,s}^\tau (\sum_{t=0}^{E-1} h_{\tau E +t} )^2+\eta_{\tau,s}^2 B_\tau \\
&+2\eta_{\tau,s} E\sum_{v=1}^V P_{v,s}^\tau (f_{i,s}(w_s^*)-f_{i,s}(\overline{w}_s^{\tau E+l}))\label{eq:bound1}
\end{align}
Define $H_\tau = (\sum_{t=0}^{E-1} h_{\tau E +t} )^2$. Now we find a bound for $H_\tau$. 
\begin{align}
\mathbb{E}_\xi [h_{\tau E+t}^2]&=\mathbb{E}_\xi \|\sum_{v=1}^V P_{v,s}^\tau g_{v,s,\tau}^t\|^2\\
&\leq \sum_{v=1}^V \frac{B_i (P_{v,s}^\tau)^2}{d_{i,s}} \mathbb{E}_\xi \|g_{v,s,\tau}^t\|^2 \\
&\leq \theta G^2\sum_{v=1}^V P_{v,s}^\tau 
\end{align}
Thus, 
\begin{align}
\mathbb{E}_\xi[H_\tau]
&=\mathbb{E}_\xi[ (\sum_{t=0}^{E-1} h_{\tau E+t})^2 ]\\
&\leq E\sum_{t=0}^{E-1}\mathbb{E}_\xi [h_{\tau E+t}^2]\\
&\leq \theta G^2 E^2 \sum_{v=1}^V P_{v,s}^\tau 
\end{align}
Notice by our definition, $B_\tau=\frac{1}{\eta_{\tau,s}^2}\mathbb{E}_\xi [Q_{\tau E+t}]$, we do not need to compute $\mathbb{E}_\xi [B_\tau]$ again ($\mathbb{E}_\xi [B_\tau]=B_\tau$). Write $\overline{\Delta}_{\tau E+t}=\mathbb{E}_\xi [\Delta_{\tau E+t}]$. By Eq. \ref{eq:bound1}, we can know
\begin{align}
(1+\frac{1}{4}\mu\eta_{\tau,s} E\sum_{v=1}^V P_{v,s}^\tau )\overline{\Delta}_{(\tau+1) E }&\leq \overline{\Delta}_{\tau E}+\frac{1}{2}\mu \eta_{\tau,s}^3 E \sum_{v=1}^V P_{v,s}^\tau \mathbb{E}_\xi[H_\tau]+\eta_{\tau,s}^2 B_\tau \\
&+2\eta_{\tau,s} E  \sum_{v=1}^V P_{v,s}^\tau (f_{i,s}(w_s^*)-f_{i,s}(\overline{w}_s^{\tau E+l}))
\end{align}
Assume $\eta_{\tau,s} \leq \frac{4}{\mu E \sum_{v=1}^V P_{v,s}^\tau}$, then
\begin{align}
(1+\frac{1}{4}\mu\eta_{\tau,s} E\sum_{v=1}^V P_{v,s}^\tau )\overline{\Delta}_{(\tau+1) E }&\leq \overline{\Delta}_{\tau E}+\eta_{\tau,s}^2 ( 2\mathbb{E}_\xi[H_\tau]+ B_\tau) \\
&+2\eta_{\tau,s} E  \sum_{v=1}^V P_{v,s}^\tau (f_{i,s}(w_s^*)-f_{i,s}(\overline{w}_s^{\tau E+l}))
\end{align}
\begin{align}
\overline{\Delta}_{(\tau+1) E }&\leq \frac{1}{1+\frac{1}{4}\mu\eta_{\tau,s} E\sum_{v=1}^V P_{v,s}^\tau}\overline{\Delta}_{\tau E}+\eta_{\tau,s}^2 ( 2\mathbb{E}_\xi[H_\tau]+ B_\tau) \\
&+2\eta_{\tau,s} E  \sum_{v=1}^V P_{v,s}^\tau (f_{i,s}(w_s^*)-f_{i,s}(\overline{w}_s^{\tau E+l})) \\
&=(1-\underbrace{\frac{\frac{1}{4}\mu\eta_{\tau,s} E\sum_{v=1}^V P_{v,s}^\tau}{1+\frac{1}{4}\mu\eta_{\tau,s} E\sum_{v=1}^V P_{v,s}^\tau}}_{\eta_{\tau,s}\leq \frac{4}{\mu E \sum_{v=1}^V P_{v,s}^\tau}})\overline{\Delta}_{\tau E}+\eta_{\tau,s}^2 ( 2\mathbb{E}_\xi[H_\tau]+ B_\tau) \\
&+2\eta_{\tau,s} E  \sum_{v=1}^V P_{v,s}^\tau (f_{i,s}(w_s^*)-f_{i,s}(\overline{w}_s^{\tau E+l})) \\
&\leq (1-\frac{1}{8}\mu\eta_{\tau,s} E\sum_{v=1}^V P_{v,s}^\tau)\overline{\Delta}_{\tau E}+\eta_{\tau,s}^2 ( 2\mathbb{E}_\xi[H_\tau]+ B_\tau) \\
&+2\eta_{\tau,s} E  \sum_{v=1}^V P_{v,s}^\tau (f_{i,s}(w_s^*)-f_{i,s}(\overline{w}_s^{\tau E+l}))\label{eq:finalbound}
\end{align}
Consider the expectation over $\mathcal{A}_{\tau,s}$. 
We know $\mathbb{E}_{\mathcal{A}_{\tau,s}}[\mathbbm{1}_{v\in \mathcal{A}_{\tau,s}}]=p_{s|v}^\tau(\pi, w_{s}^{\tau})$. 
\begin{align}
\mathbb{E}_{\mathcal{A}_{\tau,s}}[\sum_{v=1}^V P_{v,s}^\tau (f_{i,s}(w_s^*)-f_{i,s}(\overline{w}_s^{\tau E+l}))]&=
\mathbb{E}_{\mathcal{A}_{\tau,s}}[\sum_{v=1}^V \frac{d_{i,s}}{B_i p_{s|v}^\tau}\mathbbm{1}_{v\in \mathcal{A}_{\tau,s}} (f_{i,s}(w_s^*)-f_{i,s}(\overline{w}_s^{\tau E+l}))]\\
&=\sum_{v=1}^V \frac{d_{i,s}}{B_i p_{s|v}}\mathbb{E}_{\mathcal{A}_{\tau,s}}[\mathbbm{1}_{i\in \mathcal{A}_{\tau,s}} ](f_{i,s}(w_s^*)-f_{i,s}(\overline{w}_s^{\tau E+l}))\\
&=\sum_{v=1}^V \frac{d_{i,s}}{B_i}(f_{i,s}(w_s^*)-f_{i,s}(\overline{w}_s^{\tau E+l}))\\
&=\sum_{i=1}^N \sum_{b=1}^{B_i} \frac{d_{i,s}}{B_i}(f_{i,s}(w_s^*)-f_{i,s}(\overline{w}_s^{\tau E+l}))\\
&=F_s^*-F_s(\overline{w}_s^{\tau E+l}) \leq 0
\end{align}
And 
\begin{align}
\mathbb{E}_{\mathcal{A}_{\tau,s}}[\sum_{v=1}^V P_{v,s}^\tau]=\mathbb{E}_{\mathcal{A}_{\tau,s}}[\sum_{v=1}^V \frac{d_{i,s}}{B_i p_{s|v}}\mathbbm{1}_{i\in \mathcal{A}_{s,\tau}}]=1
\end{align}
Note that $\overline{w}_{s}^{\tau E}$ is the aggregation of local updates from processors in $\mathcal{A}_{\tau-1,s}$, instead of $\mathcal{A}_{\tau,s}$. Thus, $\mathbb{E}_{\mathcal{A}_{\tau,s}}(\overline{\Delta}_{\tau E})=\overline{\Delta}_{\tau E}$. Similarly, $\mathbb{E}_{\mathcal{A}_{\tau,s}}(\overline{\Delta}_{(\tau+1) E})=\overline{\Delta}_{(\tau+1) E}$. After the expectation over $\mathcal{A}_{\tau,s}$, Eq. \ref{eq:finalbound} becomes
\begin{align}
\overline{\Delta}_{(\tau+1) E }&\leq (1-\frac{1}{8}\mu\eta_{\tau,s} E)
\overline{\Delta}_{\tau E}+\eta_{\tau,s}^2 \mathbb{E}[ 2\mathbb{E}_\xi[H_\tau]+ B_\tau] 
\end{align}
Define 
\begin{align}
K_g'&= \sum_{v=1}^V (P_{v,s}^\tau)^2 \sigma_{i,s}^2+4 L \sum_{v=1}^V P_{v,s}^\tau \Gamma_{i,s}+ \frac{1}{d_s} \sum_{v=1}^V (P_{v,s}^\tau)^2 \| \overline{g}_{v,s,\tau}^0 \|^2\\
K_l'&= 
\frac{2E^2G^2\theta}{e_w^2e_f^2} V  \sum_{v=1}^V (P_{v,s}^\tau f_{v,s,\tau}^0-\frac{d_{i,s}}{B_i}f_{v,s,\tau}^0)^2
\\
K_p'&= (2+\frac{\mu}{2L}) E^2G^2 \sum_{v=1}^V P_{v,s}^\tau 
+2E^2G^2\theta (\sum_{v=1}^V P_{v,s}^\tau -1)^2
\end{align}
With Eq. \ref{eq:b_bound}, we know $B_{\tau E +t}\leq K_g'+K_l'+K_p'$.
\begin{align}
2\mathbb{E}_\xi [H_\tau]+B_\tau
&=2\mathbb{E}_\xi [H_\tau]+\sum_{t=0}^{E-1} B_{\tau E+t}\\
&\leq 
2\theta G^2 E^2 \sum_{v=1}^V P_{v,s}^\tau + E(K_g'+K_l'+K_p')
\end{align}
Define
\begin{align}
K_g&= E\sum_{v=1}^V (P_{v,s}^\tau)^2 \sigma_{i,s}^2+4 L E\sum_{v=1}^V P_{v,s}^\tau \Gamma_{i,s}+ E\frac{1}{d_s} \sum_{v=1}^V (P_{v,s}^\tau)^2 \| \overline{g}_{v,s,\tau}^0 \|^2\\
K_l&= \frac{2E^3G^2\theta}{e_w^2e_f^2} V  \sum_{v=1}^V (P_{v,s}^\tau f_{v,s,\tau}^0-\frac{d_{i,s}}{B_i}f_{v,s,\tau}^0)^2\\
K_p&= (2\theta+E(2+\frac{\mu}{2L})) E^2G^2 \sum_{v=1}^V P_{v,s}^\tau 
+2E^3G^2\theta (\sum_{v=1}^V P_{v,s}^\tau -1)^2
\end{align}
Then we know
\begin{align}
\mathbb{E}_{\mathcal{A}_{\tau,s}}[2\mathbb{E}_\xi [H_\tau]+B_\tau] \leq \mathbb{E}_{\mathcal{A}_{\tau,s}}[K_g+K_l+K_p]
\end{align}
Notice
\begin{align}
\mathbb{E}_{\mathcal{A}_{\tau,s}}[\sum_{v=1}^V (P_{v,s}^\tau)^2]=\mathbb{E}[\sum_{v=1}^V (\frac{(d_{i,s})^2}{(B_i p_{s|v}^\tau)^2}) \mathbbm{1}_{v\in\mathcal{A}_{\tau,s}}\times \mathbbm{1}_{v\in\mathcal{A}_{\tau,s}}]
=\sum_{v=1}^V \frac{(d_{i,s})^2}{(B_i)^2p_{s|v}^\tau}
\end{align}
\begin{align}
\mathbb{E}_{\mathcal{A}_{\tau,s}}[K_g]&=E\sum_{v=1}^V \frac{(\frac{d_{i,s}}{B_i}\sigma_{i,s})^2}{p_{s|v}^\tau}+4LE\sum_{v=1}^V \frac{d_{i,s}}{B_i} \Gamma_{i,s}+\frac{E}{d_s} \sum_{v=1}^V \frac{\| \frac{d_{i,s}}{B_i} \overline{g}_{v,s,\tau}^0 \|^2}{p_{s|v}}\\
\mathbb{E}_{\mathcal{A}_{\tau,s}}[K_l]&=\frac{2E^3G^2\theta}{e_w^2e_f^2} \mathbb{E}[V  \sum_{v=1}^V (P_{v,s}^\tau f_{v,s,\tau}^0-\frac{d_{i,s}}{B_i}f_{v,s,\tau}^0)^2]\\
\mathbb{E}_{\mathcal{A}_{\tau,s}}[K_p]&=(2\theta+E(2+\frac{\mu}{2L})) E^2G^2 
+2E^3G^2\theta \mathbb{E}[(\sum_{v=1}^V P_{v,s}^\tau -1)^2]
\end{align}

\begin{align}
\overline{\Delta}_{(\tau+1) E }\leq
(1-\frac{1}{8}\mu\eta_{\tau,s} E)
\overline{\Delta}_{\tau E}+\eta_{\tau,s}^2 \mathbb{E}_{\mathcal{A}_{\tau,s}}[K_g+K_l+K_p] 
\end{align}
Write $K_\tau=\mathbb{E}_{\mathcal{A}_{\tau,s}}[K_g+K_l+K_p]$

\textbf{Proof of the convergence theorem}

Define $\gamma=\max \{\frac{32L}{\mu},4E\sum_{v=1}^V P_{v,s}^\tau\}$, $V_\tau=\max\{\gamma^2 \mathbb{E}\|w_s^0-w^*\|^2, (\frac{16}{\mu})^2\sum_{\tau'=0}^{\tau-1}K_{\tau'}\}$
Now we show by induction that we can obtain
\begin{align}
\overline{\Delta}_{\tau E} \leq \frac{V_\tau}{(\tau E +\gamma)^2}
\end{align}
Let $\eta_{\tau,s}=\frac{16}{\mu } \frac{1}{(\tau+1)E+\gamma}$. Initially, $\frac{V_0}{\gamma^2}\geq \overline{\Delta}_0$. Suppose $\overline{\Delta}_{\tau E} \leq \frac{V_\tau}{(\tau E +\gamma)^2}$, then we show $\overline{\Delta}_{(\tau+1) E} \leq \frac{V_{\tau+1}}{((\tau+1) E +\gamma)^2}$
\begin{align}
\overline{\Delta}_{(\tau+1)E}
&\leq (1-\frac{2E}{(\tau+1)E+\gamma})\overline{\Delta}_{\tau E}+\eta_{\tau,s}^2 K_\tau \\
&= (\frac{\tau E+\gamma -E}{(\tau+1)E+\gamma})\overline{\Delta}_{\tau E}+\eta_{\tau,s}^2 K_\tau \\
&\leq (\frac{\tau E+\gamma -E}{(\tau+1)E+\gamma})\frac{V_\tau}{(\tau E +\gamma)^2}+(\frac{16}{\mu})^2 \frac{K_\tau}{((\tau+1)E+\gamma)^2} \\
&= \frac{(\tau-1) E+\gamma}{(\tau+1)E+\gamma}\frac{V_\tau}{(\tau E +\gamma)^2}+ \frac{(\frac{16}{\mu})^2K_\tau}{((\tau+1)E+\gamma)^2} \\
&= \frac{(\tau-1) E+\gamma}{(\tau E +\gamma)^2}\frac{V_\tau}{(\tau+1)E+\gamma}+ \frac{(\frac{16}{\mu})^2K_\tau}{((\tau+1)E+\gamma)^2}
\end{align}
Notice
\begin{align}
\frac{(\tau-1)E+\gamma}{(\tau E+\gamma)^2}
\leq \frac{(\tau-1)E+\gamma}{(\tau E+\gamma)^2-E^2}=\frac{1}{(\tau+1)E+\gamma}
\end{align}
Thus, 
\begin{align}
\overline{\Delta}_{(\tau+1)E}\leq 
\frac{V_\tau+(\frac{16}{\mu})^2K_\tau}{((\tau+1)E+\gamma)^2}=\frac{V_{\tau+1}}{((\tau+1)E+\gamma)^2}
\end{align}
We assumed $\eta_{\tau,s} \leq \frac{1}{2L}$ and $\eta_{\tau,s} \leq \frac{4}{\mu E \sum_{v=1}^V P_{v,s}^\tau}$ in the proof. We can check $\eta_{\tau,s}=\frac{16}{\mu } \frac{1}{(\tau+1)E+\gamma}$ satisfies these assumptions. 
\begin{align}
\eta_{\tau,s} &\leq \eta_0 = \frac{16}{\mu}\frac{1}{E+\gamma}
\leq \frac{16}{\mu}\underbrace{\frac{1}{\gamma}}_{\gamma\geq \frac{32L}{\mu}}\leq \frac{1}{2L}\\
\eta_{\tau,s} &\leq \eta_0 = \frac{16}{\mu}\frac{1}{E+\gamma}
\leq \frac{16}{\mu}\frac{1}{\gamma}\leq \frac{16}{\mu}\frac{1}{ 4E\sum_{v=1}^V P_{v,s}^\tau}=\frac{4}{\mu E \sum_{v=1}^V P_{v,s}^\tau}
\end{align}


\begin{theorem}[Optimal assignment probabilities] \label{theorem:solutionOS}
Consider the optimization problem: 
\begin{align}
\label{eq:OSproblem}
&\min_{\{p_{s|(i,b)}^\tau \}} \; \sum_{s=1}^{S} \sum_{i\in \mathcal{N}_s}\sum_{b=1}^{B_i} \frac{\| \frac{d_{i,s}}{B_i \eta_{\tau,s}} G_{(i,b),s}^\tau \|^2}{p_{s|(i,b)}^\tau} \\
&\text{s.t.} \;  p_{s|(i,b)}^\tau \geq 0,\; \sum_{s\in\mathcal{S}_i} p_{s|(i,b)}^\tau \leq 1,\; 
\sum_{s=1}^S\sum_{i\in\mathcal{N}_s}\sum_{b=1}^{B_i} p_{s|(i,b)}^\tau = m, \nonumber \\ 
& \quad \forall i,b,s,\tau, \nonumber
\end{align}
Equation \eqref{eq:OSproblem}'s optimization problem is solved by setting: 
\begin{align}
p_{s|(i,b)}^\tau =
\begin{cases}
\frac{(m-V+k)\|\tilde{U}_{(i,b),s}^\tau\|}{\sum_{(j,v) \in \mathcal{V}_{0}} M_{(j,v)}^\tau} & (i,b) \in \mathcal{V}_{0}, \\
\frac{\|\tilde{U}_{(i,b),s}^\tau\|}{M_{(i,b)}^\tau} & (i,b) \notin \mathcal{V}_{0},
\end{cases}
\label{eq:solution}
\end{align}
where $V=\sum_{i\in\mathcal{N}_1 \cup\cdots \mathcal{N}_S} B_i$, $\tilde{U}_{(i,b),s}^\tau = \frac{d_{i,s}}{B_i \eta_{\tau,s}} G_{(i,b),s}^\tau$, $M_{(i,b)}^\tau = \sum_{s \in \mathcal{S}_i} \|\tilde{U}_{(i,b),s}^\tau\|$.  $\mathcal{V}_{0}$ is the largest set satisfying
\begin{align}
0 < (m-V+k) \leq \frac{\sum_{(j,v) \in \mathcal{V}_{0}} M_{(j,v)}^\tau}{\max_{(i,b) \in \mathcal{V}_{0}} [M_{(i,b)}^\tau]}\nonumber
\end{align}
where $k = |\mathcal{V}_{0}|$. 
\end{theorem}

\begin{proof}
As Theorem \ref{thoerem:convergence}, we firstly reindex processor $(i,b)$ to $v=1,2,\dots,V$ for convenience. Define $\mathcal{V}_s$ as the set of available $\textit{processors}$ for model $s$. 
The problem can be written as follows. 
\begin{equation}
\begin{aligned}
& \underset{\{p_{s|v}^\tau\}}{\text{min}}
& & \sum_{s=1}^{S} \sum_{v\in\mathcal{V}_s} \frac{\| \tilde{U}_{v,s} \|^2}{p_{s|v}} \\
& \text{s.t.}
& & p_{s|v}^\tau \geq 0, \; \forall v = 1, \cdots, V, \; s = 1, \cdots, S, \\
&&& \sum_{s\in\mathcal{S}_v} p_{s|v}^\tau \leq 1, \; \forall v = 1, \cdots, V, \\
&&& \sum_{s=1}^{S} \sum_{v\in\mathcal{V}_s} p_{s|v}^\tau \leq m,
\end{aligned}
\end{equation}
The Lagrangian function of this optimization problem is 
\begin{align}
L=\sum_{s=1}^S \sum_{v\in\mathcal{V}_s} \frac{\|\tilde{U}_{v,s}^\tau\|^2}{p_{s|v}^\tau}
-\sum_{s=1}^S \sum_{v\in\mathcal{V}_s}  \lambda_{v,s}^\tau p_{s|v}^\tau
-\sum_{v=1}^V \rho_v^\tau \left(1-\sum_{s\in\mathcal{S}_v}  p_{s|v}^\tau \right)
-y \left(m- \sum_{s=1}^S \sum_{v\in\mathcal{V}_s}
 p_{s|v}^\tau\right)
\end{align}
KKT conditions: 
\begin{subequations}
\begin{align}
    \frac{\partial L}{\partial p_{s|v}^\tau} &= 0, \quad \forall v=1,\cdots,V,\; s=1,\cdots, S \\
    \lambda_{v,s}^\tau &\geq 0, \quad p_{s|v}^\tau \geq 0, \quad \lambda_{v,s}^\tau p_{s|v}^\tau = 0, \quad \forall v=1,\cdots,V,\; s=1,\cdots, S \\
    \rho_v^\tau &\geq 0, \quad \sum_{s\in\mathcal{S}_v} p_{s|v}^\tau \leq 1, \quad \rho_v^\tau \left(1 - \sum_{s\in\mathcal{S}_v}  p_{s|v}^\tau\right) = 0, \quad \forall v=1,\cdots,V \\
    y &\geq 0, \quad  \sum_{s=1}^{S}\sum_{v\in\mathcal{V}_s} p_{s|v}^\tau \leq m, \quad y\left(m - \sum_{s=1}^{S} \sum_{v\in\mathcal{V}_s} p_{s|v}^\tau\right) = 0
\end{align}
\end{subequations}
Based on KKT conditions, we derive the optimal solution. 
Firstly, based on the derivative:
\begin{align}
p_{s|v}^\tau = \frac{1}{\sqrt{-\lambda_{v,s}^\tau+\rho_v^\tau+y}}\|\tilde{U}_{v,s}^\tau\|
\end{align}
Consider $\lambda_{v,s}^\tau p_{s|v}^\tau=0$, we know $\lambda_{v,s}^\tau=0$ for all $v,s$. Thus,
\begin{align}
p_{s|v}^\tau = \frac{1}{\sqrt{\rho_v^\tau+y}}\|\tilde{U}_{v,s}^\tau\|
\end{align}
If there exists $v$, for which $\rho_v^\tau\neq 0$, then $1-\sum_{s\in\mathcal{S}_v} p_{s|v}^\tau=0$. 
\begin{align}
\sum_{s\in\mathcal{S}_v} \frac{1}{\sqrt{\rho_v^\tau +y}} \|\tilde{U}_{v,s}^\tau\|=1
\end{align}
\begin{align}
\frac{1}{\sqrt{\rho_v^\tau+y}}=\frac{1}{\sum_{s\in\mathcal{S}_v} \|\tilde{U}_{v,s}^\tau\| }
\end{align}
In this case, 
\begin{align}
p_{s|v}^\tau=\frac{\|\tilde{U}_{v,s}^\tau\|}{\sum_{s\in\mathcal{S}_v} \|\tilde{U}_{v,s}^\tau\|}
\end{align}
For the rest $i$, $\rho_v^\tau=0$, 
\begin{align}
p_{s|v}^\tau=\frac{1}{\sqrt{y}}\|\tilde{U}_{v,s}^\tau\|
\end{align}
Now consider $y(m-\sum_{s=1}^S \sum_{v\in\mathcal{V}_s} p_{s|v}^\tau)=0$. It is clear that when $m<V$, there exists at least one $v$ satisfying $\rho_v^\tau=0$, otherwise $\sum_{s=1}^S \sum_{v\in\mathcal{V}_s} p_{s|v}^\tau=V > m$ (KKT condition breaks). And if we have $\rho_v^\tau=0$, then we have $p_{s|v}^\tau=\frac{1}{\sqrt{y}}\|\tilde{U}_{v,s}^\tau\|$, so $y$ cannot be 0, there must be $m-\sum_{s=1}^S \sum_{v\in\mathcal{V}_s} p_{s|v}^\tau=0$. 
When $m=V$, it is full participation, $\rho_v^\tau\neq 0$ for all $v$, and $y$ can be any value. 

For the convenience of proof, we reorder all \textit{processors} to ensure that for $v=1,\cdots, k$, we have $\rho_v^\tau=0$, $p_{s|v}^\tau=\frac{1}{\sqrt{y}}\|\tilde{U}_{v,s}^\tau\|$; for $v=k+1,\cdots,V$, we have $\rho_v^\tau \neq 0$, $\sum_{s\in\mathcal{S}_v} p_{s|v}^\tau=1$. Later we will show that the index order of \textit{processors} does not influence the proof. 
Now based on KKT conditions, we know:
\begin{align}
m-\sum_{s=1}^S \sum_{v\in\mathcal{V}_s} p_{s|v}^\tau=0\\
m-\sum_{v=1}^V \sum_{s\in\mathcal{S}_v} p_{s|v}^\tau=0
\\
m- \sum_{v=1}^k \sum_{s\in\mathcal{S}_v} \frac{\|\tilde{U}_{v,s}^\tau\|}{\sqrt{y}}-
\sum_{v=k+1}^V 1
=0
\end{align}
Thus, 
\begin{align}
\frac{1}{\sqrt{y}}=(m-V+k)\frac{1}{ \sum_{v=1}^k \sum_{s\in\mathcal{S}_v} \|\tilde{U}_{v,s}^\tau\|}
\end{align}
\begin{align}
p_{s|v}^\tau=(m-V+k)\frac{\|\tilde{U}_{v,s}^\tau\|}{\sum_{j=1}^k \sum_{s\in\mathcal{S}_j} \|\tilde{U}_{j,s}^\tau\|}
\end{align}
Based on KKT conditions, we also need to ensure that
\begin{align}
0<\sum_{s\in\mathcal{S}_v} p_{s|v}^\tau \leq 1 
\end{align}
For  $v=k+1,\cdots,V$, this always holds. For $v=1,\cdots,k$, this inequality can be written as 
\begin{align}
0< \sum_{s\in\mathcal{S}_v} (m-V+k)\frac{\|\tilde{U}_{v,s}^\tau\|}{\sum_{j=1}^k \sum_{s\in\mathcal{S}_j} \|\tilde{U}_{j,s}^\tau\|} \leq 1 
\end{align}
\begin{align}
0< (m-V+k)\frac{\sum_{s\in\mathcal{S}_v} \|\tilde{U}_{v,s}^\tau\|}{\sum_{j=1}^k\sum_{s\in\mathcal{S}_j}\|\tilde{U}_{j,s}^\tau\|} \leq 1 
\end{align}
\begin{align}
0< (m-V+k) \leq \frac{\sum_{j=1}^k\sum_{s\in\mathcal{S}_j} \|\tilde{U}_{j,s}^\tau\|}{\sum_{s\in\mathcal{S}_v} \|\tilde{U}_{v,s}^\tau\|}
\end{align}
Let $M_v^\tau=\sum_{s\in\mathcal{S}_v} \|\tilde{U}_{v,s}^\tau\|$. This inequality can be written as
\begin{align}
0< m-V+k \leq \frac{\sum_{j=1}^k M_j^\tau}{M_v^\tau},\; \forall v=1,\cdots,k
\label{ineq}
\end{align}
Now go back to evaluate the objective function: 
\begin{align}
F(k)=\sum_{v=1}^V \sum_{s\in\mathcal{S}_v} \frac{\|\tilde{U}_{v,s}^\tau\|^2}{p_{s|v}^\tau}
&=\sum_{v=1}^k \sum_{s\in\mathcal{S}_v}
 \frac{\|\tilde{U}_{v,s}^\tau\|}{m-V+k} 
\left(\sum_{v=1}^k \sum_{s\in\mathcal{S}_v}\|\tilde{U}_{v,s}^\tau\|\right)+
\sum_{v=k+1}^V \sum_{s\in\mathcal{S}_v}\left(\sum_{s\in\mathcal{S}_v} \|\tilde{U}_{v,s}^\tau\|\right) \|\tilde{U}_{v,s}^\tau\|
\\
&=\frac{1}{m-V+k} \left(\sum_{v=1}^k \sum_{s\in\mathcal{S}_v} \|\tilde{U}_{v,s}^\tau\|\right)^2 + \sum_{v=k+1}^V \left(\sum_{s\in\mathcal{S}_v} \|\tilde{U}_{v,s}^\tau\|\right)^2\\
&= \frac{1}{m-V+k} \left(\sum_{v=1}^k M_v^\tau \right)^2 + \sum_{v=k+1}^V (M_v^\tau)^2
\end{align}
Next, we will prove that $F(k)$ is monotonically decreasing with respect to $k$. Therefore, to minimize $F(k)$, we need the largest $k$. 

Consider 
\begin{align}
G(k)=F(k) - \sum_{v=1}^V (M_v^\tau)^2
=\frac{1}{m-V+k}\left(\sum_{v=1}^k M_v^\tau \right)^2 -\sum_{v=1}^k (M_v^\tau)^2
\end{align}
Now we show $G(k)-G(k+1)\geq 0$. 
\begin{align}
G(k)-G(k+1)&=\left[
\frac{1}{m-V+k}\left(\sum_{v=1}^k M_v^\tau \right)^2 -\sum_{v=1}^k (M_v^\tau)^2
\right]\\
~& - \left[
\frac{1}{m-V+k+1}\left(\sum_{v=1}^{k+1} M_v^\tau \right)^2 -\sum_{v=1}^{k+1} (M_v^\tau)^2
\right]
\end{align}
Define $C=m-V+k$. This can be written as: 
\begin{align}
G(k)-G(k+1)=\frac{1}{C(C+1)}\left[
(C+1)\left(\sum_{v=1}^{k} M_v^\tau \right)^2
-C\left(\sum_{v=1}^{k+1} M_v^\tau \right)^2
+C(C+1) (M_{k+1}^\tau)^2
\right]\label{eq:G}
\end{align}
$\frac{1}{C(C+1)}$ is always postive. Eq. \ref{eq:G} can be reorganized as
\begin{align}
C(C+1)[G(k)-G(k+1)]&=
C\left(\sum_{v=1}^{k} M_v^\tau \right)^2+\left(\sum_{v=1}^{k} M_v^\tau \right)^2
-C\left(\sum_{v=1}^{k} M_v^\tau + M_{k+1}^\tau \right)^2
+C(C+1) (M_{k+1}^\tau)^2 \\
&=\left(\sum_{v=1}^{k} M_v^\tau \right)^2-C (M_{k+1}^\tau)^2-2C \left(\sum_{v=1}^{k} M_v^\tau \right) M_{k+1}^\tau
+C(C+1) (M_{k+1}^\tau)^2 \\
&=\left(\sum_{v=1}^{k} M_v^\tau \right)^2-2C \left(\sum_{v=1}^{k} M_v^\tau \right) M_{k+1}^\tau
+C^2 (M_{k+1}^\tau)^2 \\
&=\left[
\sum_{v=1}^k M_v^\tau - C M_{k+1}^\tau
\right]^2 \geq 0 \label{eq:final}
\end{align}
With Eq. \ref{eq:final}, we prove that $G(k)\geq G(k+1)$, then we also know $F(k) \geq F(k+1)$. 
Given the closed-form solution of $p_{s|v}^\tau$: 
\begin{align}
p_{s|v}^\tau =
\begin{cases}
(m-V+k)\frac{\|\tilde{U}_{v,s}^\tau\|}{\sum_{j=1}^k M_j^\tau} & \text{if } v=1,2,\cdots,k ,\\
\frac{\|\tilde{U}_{v,s}^\tau\|}{M_i^\tau} & \text{if } v=k+1,\cdots,V
\end{cases}
\label{solution}
\end{align}
To minimize the objective $\sum_{i=1}^N \sum_{s\in\mathcal{S}_v} \frac{\|\tilde{U}_{v,s}^\tau\|^2}{p_{s|v}^\tau}$, The solution requires the largest $k$. And $k$ should satisfy the inequality: 
\begin{align}
0< m-V+k \leq \frac{\sum_{j=1}^k M_j^\tau}{M_v^\tau},\; \forall v=1,\cdots,k
\label{ineq}
\end{align}
Initially, assume $k=V$. If the inequality holds for all $v=1,\cdots,V$, then the largest $k=V$. \\
Otherwise, there should be at least one $v$ that breaks the inequality. Assume that only one $v=j_{V}$ breaks the inequality, $M_{j_{V}}^\tau$ must be the largest among all $M_{v}^\tau$. Let $k=V-1$, and switch the index order between $j_V$ and $V$ so that $M_V^\tau$ is the largest among all. If the inequality holds for all $v=1,\cdots, V-1$, then the largest $k=V-1$. \\
Otherwise, there should be at least one $v$ among $v=1,\cdots,V-1$ that breaks the inequality. Assume that only one $v=j_{V-1}$ breaks the inequality, $M_{j_{V-1}}^\tau$ must be the largest among all $M_{v}^\tau$ for $v=1,\cdots, V-1$. Let $k=V-2$, and switch the index order between $j_{V-1}$ and $V-1$. If the inequality holds for all $v=1,\cdots, V-2$, then the largest $k=V-2$. \\
This process continues iteratively until we find the $k$ that satisfies the inequality for all $v=1,\cdots, k$. 
We group $v=1,\dots,k$ as set $\mathcal{V}_{0}$. 
\end{proof}

\begin{theorem}[Optimal MMFL-LVR assignment probabilities]
\label{theorem:solutionAS}
Consider the optimization problem: 
\begin{align}
\label{eq:ASproblem}
&\min_{\mathbf{p}^\tau} \; \sum_{s=1}^{S} \mathbb{E}[(\sum_{(i,b)\in\mathcal{A}_{\tau,s}} P_{(i,b),s}^\tau f_{i,s}(w_s^\tau) -\sum_{i\in\mathcal{N}_s} d_{i,s} f_{i,s}(w_s^\tau))^2] \\
&\text{s.t.}\;  p_{s|(i,b)}^\tau \geq 0,\; \sum_{s=1}^S p_{s|(i,b)}^\tau \leq 1,\; \nonumber \sum_{s=1}^S\sum_{i\in\mathcal{N}_s}\sum_{b=1}^{B_i} p_{s|(i,b)}^\tau = m, \\
&\; \quad \forall i,b,s,\tau. \nonumber
\end{align}
Equation (\ref{eq:ASproblem})'s optimization problem is solved by
\begin{align}
p_{s|(i,b)}^\tau =
\begin{cases}
\frac{(m-V+k)\|\tilde{U}_{(i,b),s}^\tau\|}{\sum_{(j,v) \in \mathcal{V}_{0}} M_{(j,v)}^\tau} & \text{if } (i,b) \in \mathcal{V}_{0}, \\
\frac{\|\tilde{U}_{(i,b),s}^\tau\|}{M_{(i,b)}^\tau} & \text{if } (i,b) \notin \mathcal{V}_{0}.
\end{cases}
\label{eq:solutionAS}
\end{align}
where $V=\sum_{i\in\mathcal{N}_1 \cup\cdots \mathcal{N}_S} B_i$, $\tilde{U}_{(i,b),s}^\tau = \frac{d_{i,s}}{B_i} f_{i,s}(w_s^\tau)$, $M_{(i,b)}^\tau = \sum_{s \in \mathcal{S}_i} \|\tilde{U}_{(i,b),s}^\tau\|$, and $m$ is the number of active processors on expectation. $\mathcal{V}_{0}$ is the largest set satisfying
\begin{align}
0 < (m-V+k) \leq \frac{\sum_{(j,v) \in \mathcal{V}_{0}} M_{(j,v)}^\tau}{\max_{(i,b) \in \mathcal{V}_{0}} [M_{(i,b)}^\tau]}\nonumber
\end{align}
where $k = |\mathcal{V}_{0}|$. 
\end{theorem}

We omit the detailed proof of Theorem \ref{theorem:solutionAS} as it can be solved following similar steps as Theorem \ref{theorem:solutionOS}.

\begin{theorem}[MMFL-StaleVR optimal solution] \label{theorem:solutionStale}
Consider the optimization problem: 
\begin{align}
&\min_{\mathbf{p}^\tau,
\{\beta_{(i,b),s}^\tau\}} \sum_{s=1}^S \mathbb{E}\left[\| \frac{\Delta_{\tau,s}}{\eta_{\tau,s}}-\sum_{i\in\mathcal{N}_s}\sum_{b=1}^{B_i} \frac{d_{i,s}}{B_i} \frac{G_{(i,b),s}^\tau}{\eta_{\tau,s}}\|^2\right] \label{eq:Staleproblem}\\
&\text{s.t.}\;  p_{s|(i,b)}^\tau \geq 0,\; \sum_{s=1}^S p_{s|(i,b)}^\tau \leq 1,\; \nonumber \sum_{s=1}^S\sum_{i\in\mathcal{N}_s}\sum_{b=1}^{B_i} p_{s|(i,b)}^\tau = m, \\
&\; \quad \forall i,b,s,\tau. \nonumber
\end{align}
Equation \eqref{eq:Staleproblem}'s optimization problem is solved by setting:
\begin{align}
\beta_{(i,b),s}^\tau&= \frac{(G_{(i,b),s}^\tau)^\top h_{i,s}^\tau}{\|h_{i,s}^\tau\|^2}\label{eq:solutionStaleb}\\
p_{s|(i,b)}^\tau&=
\begin{cases}
\frac{(m-V+k)\|\tilde{U}_{(i,b),s}^\tau\|}{\sum_{(j,v) \in \mathcal{V}_{0}} M_{(j,v)}^\tau} & (i,b) \in \mathcal{V}_{0}, \\
\frac{\|\tilde{U}_{(i,b),s}^\tau\|}{M_{(i,b)}^\tau} & (i,b) \notin \mathcal{V}_{0},
\end{cases}\label{eq:solutionStalep}
\end{align}
where $V=\sum_{i\in\mathcal{N}_1 \cup\cdots \mathcal{N}_S} B_i$, $\tilde{U}_{(i,b),s}^\tau = \frac{d_{i,s}(G_{(i,b),s}^\tau-h_{i,s}^\tau)}{B_i \eta_{\tau,s}}$, $M_{(i,b)}^\tau = \sum_{s \in \mathcal{S}_i} \|\tilde{U}_{(i,b),s}^\tau\|$. $\mathcal{V}_{0}$ is the largest set satisfying:
\begin{align}
0 < (m-V+k) \leq \frac{\sum_{(j,v) \in \mathcal{V}_{0}} M_{(j,v)}^\tau}{\max_{(i,b) \in \mathcal{V}_{0}} [M_{(i,b)}^\tau]}
\end{align}
where $k = |\mathcal{V}_{0}|$. 
\end{theorem}
\begin{proof}
Equation \eqref{eq:Staleproblem} can be written as: 
\begin{align}
&\sum_{s=1}^S \mathbb{E}\left[\left\| \frac{\Delta_{\tau,s}}{\eta_{\tau,s}}-\sum_{i\in\mathcal{N}_s}\sum_{b=1}^{B_i} \frac{d_{i,s}}{B_i} \frac{G_{(i,b),s}^\tau}{\eta_{\tau,s}}\right\|^2\right]\\
&=\sum_{s=1}^S \mathbb{E}\left[\left\|\frac{1}{\eta_{\tau,s}} \left(\Delta_{\tau,s}-\sum_{i\in\mathcal{N}_s}\sum_{b=1}^{B_i} \frac{d_{i,s}}{B_i} G_{(i,b),s}^\tau\right)\right\|^2\right]\\
&=\sum_{s=1}^S \mathbb{E}\left[\bigg\|\frac{1}{\eta_{\tau,s}} \bigg(\sum_{i\in\mathcal{N}_s} \sum_{b=1}^{B_i} \frac{d_{i,s} \beta_{(i,b),s}^\tau}{B_i} h_{i,s}^\tau+\sum_{(i,b)\in\mathcal{A}_{\tau,s}}d_{i,s} \frac{ G_{(i,b),s}^\tau-\beta_{(i,b),s}^\tau h_{i,s}^\tau }{B_i p_{s|(i,b)}^\tau}
-\sum_{i\in\mathcal{N}_s}\sum_{b=1}^{B_i} \frac{d_{i,s}}{B_i} G_{(i,b),s}^\tau\bigg)\bigg\|^2\right] \label{eq:244}\\
&=\sum_{s=1}^S \mathbb{E}\left[\bigg\|\frac{1}{\eta_{\tau,s}} \bigg(\sum_{(i,b)\in\mathcal{A}_{\tau,s}}d_{i,s} \frac{ G_{(i,b),s}^\tau-\beta_{(i,b),s}^\tau h_{i,s}^\tau }{B_i p_{s|(i,b)}^\tau}
-\sum_{i\in\mathcal{N}_s}\sum_{b=1}^{B_i} \frac{d_{i,s}}{B_i} (G_{(i,b),s}^\tau-\beta_{(i,b),s}^\tau h_{i,s}^\tau)   \bigg)\bigg\|^2\right]\\
&\text{Define $Z_{(i,b),s}^\tau=G_{(i,b),s}^\tau-\beta_{(i,b),s}^\tau h_{i,s}^\tau$}\\
&=\sum_{s=1}^S \mathbb{E}\left[\bigg\|\frac{1}{\eta_{\tau,s}} \bigg(\sum_{(i,b)\in\mathcal{A}_{\tau,s}}d_{i,s} \frac{ Z_{(i,b),s}^\tau }{B_i p_{s|(i,b)}^\tau}
-\sum_{i\in\mathcal{N}_s}\sum_{b=1}^{B_i} \frac{d_{i,s}}{B_i} Z_{(i,b),s}^\tau   \bigg)\bigg\|^2\right]
\\
&=\sum_{s=1}^S \frac{1}{(\eta_{\tau,s})^2} \left[\mathbb{E}\left[\bigg\| \sum_{(i,b)\in\mathcal{A}_{\tau,s}}d_{i,s} \frac{ Z_{(i,b),s}^\tau }{B_i p_{s|(i,b)}^\tau}
  \bigg\|^2\right] - \left\|\sum_{i\in\mathcal{N}_s}\sum_{b=1}^{B_i} \frac{d_{i,s}}{B_i} Z_{(i,b),s}^\tau\right\|^2\right]
\\
&=\sum_{s=1}^S \frac{1}{(\eta_{\tau,s})^2} \bigg[  \sum_{(i,b)\in\mathcal{A}_{\tau,s}} (d_{i,s})^2 \frac{ \|Z_{(i,b),s}^\tau\|^2 }{(B_i)^2 p_{s|(i,b)}^\tau} 
+ \sum_{(i,b)\neq (j,c)} d_{i,s} d_{j,s} \frac{ (Z_{(i,b),s}^\tau)^\top Z_{(j,c),s}^\tau }{B_i B_j } \\
&- \sum_{(i,b),(j,c)} \frac{d_{i,s} d_{j,s}}{B_i B_j} (Z_{(i,b),s}^\tau)^\top Z_{(j,c),s}^\tau
  \bigg]
\\
&=\sum_{s=1}^S \sum_{i\in\mathcal{N}_s}\sum_{b=1}^{B_i} \frac{1-p_{s|(i,b)}^\tau}{p_{s|(i,b)}^\tau} \left\|\frac{d_{i,s}}{\eta_{\tau,s} B_i}
(G_{(i,b),s}^\tau-\beta_{(i,b),s}^\tau h_{(i,b),s}^\tau
)
\right\|^2 \label{eq:245}
\end{align}
We have two groups of hyperparameters to optimize: sampling distribution $\mathbf{p}^\tau$ and staleness coefficients $\{\beta_{(i,b),s}^\tau\}_{s\in \mathcal{S},i\in \mathcal{N}_s}$. There is no constraints on $\beta_{(i,b),s}^\tau$.  Notice that the choice of $\mathbf{p}^\tau$ has no impact on the optimization of $\{\beta_{(i,b),s}^\tau\}_{s\in \mathcal{S},i\in \mathcal{N}_s}$. The optimal solution of $\{\beta_{(i,b),s}^\tau\}_{s\in \mathcal{S},i\in \mathcal{N}_s}$ to minimize Eq. \eqref{eq:245} is equivalent to minimizing $\|G_{(i,b),s}^\tau - \beta_{(i,b),s}^\tau h_{i,s}^\tau\|^2$ for all processors $(i,b)$. The optimal solution is straightforward:  
\begin{align}
\beta_{(i,b),s}^\tau&= \frac{(G_{(i,b),s}^\tau)^\top h_{i,s}^\tau}{\|h_{i,s}^\tau\|^2}
\end{align}
After determining $\{\beta_{(i,b),s}^\tau\}_{s\in \mathcal{S},i\in \mathcal{N}_s}$, the optimization of sampling distribution $\mathbf{p}^\tau$ follows the similar steps as Theorem \ref{theorem:solutionOS}. 
\end{proof}

\end{document}